\documentclass[11pt]{article}
\usepackage[margin=1.2in]{geometry}
\usepackage{braket,amsfonts}
\usepackage{pgfplots}

\usepackage{microtype}
\usepackage{graphicx}
\usepackage{subfigure}
\usepackage{booktabs} % for professional tables
\usepackage{hyperref}
\usepackage{amssymb,amsmath,amsthm,mathtools, esint}
\usepackage[capitalize,noabbrev]{cleveref}
\usepackage{array}
\usepackage{algorithmic}

\usepackage{graphicx,epstopdf}

\Crefname{ALC@unique}{Line}{Lines}

\usepackage{amsopn}

\usepackage{graphicx}
\graphicspath{{figures/}}
\RequirePackage{fancyhdr}
\RequirePackage{xcolor} % changed from color to xcolor (2021/11/24)
\RequirePackage{algorithm}
\RequirePackage{algorithmic}
\usepackage{tikz}

\usepackage{multirow}
\usepackage{float}
\usepackage{caption}
\usepackage{rotating}

\DeclareMathOperator{\Exp}{\mathop{\mathbb{E}}}

\usepackage{amsmath,amsfonts,bm}
\usepackage{xcolor}
\def\eqref#1{(\ref{#1})}
\def\1{\bm{1}}

\def\vtheta{{\bm{\theta}}}

\DeclareMathAlphabet{\mathsfit}{\encodingdefault}{\sfdefault}{m}{sl}
\SetMathAlphabet{\mathsfit}{bold}{\encodingdefault}{\sfdefault}{bx}{n}

\newcommand{\R}{\mathbb{R}}

\usepackage{xspace}
\usepackage{bold-extra}
\usepackage[most]{tcolorbox}

\colorlet{texcscolor}{blue!50!black}
\colorlet{texemcolor}{red!70!black}
\colorlet{texpreamble}{red!70!black}
\colorlet{codebackground}{black!25!white!25}

\numberwithin{equation}{section}

\newtheorem{theorem}{Theorem}
\newtheorem{proposition}{Proposition}

\newtheorem{lemma}{Lemma}
\newtheorem{remark}{Remark}
\newtheorem{definition}{Definition}

\numberwithin{equation}{section}
\newenvironment{keywords}{
    \begin{trivlist}\item[]{\bfseries Keywords:}\ }
    {\end{trivlist}}

\author{Wonjun Lee\thanks{Institute for Mathematics and Its Applications, University of Minnesota, USA (\texttt{lee01273@umn.edu}).}
\and Yifei Yang\thanks{Electronic Information School, Wuhan University, China (\texttt{yfyang@whu.edu.cn}).}
\and Dongmian Zou\thanks{Division of Natural and Applied Sciences, Duke Kunshan University, China (\texttt{dongmian.zou@duke.edu}).}
\and Gilad Lerman\thanks{School of Mathematics, University of Minnesota, USA (\texttt{lerman@umn.edu}).}}

\title{Monotone Generative Modeling \\ via a Gromov-Monge Embedding}

\begin{document}

\maketitle

\begin{abstract}
Generative adversarial networks (GANs) are popular for generative tasks; however, they often require careful architecture selection, extensive empirical tuning, and are prone to mode collapse. 
To overcome these challenges, we propose a novel model that identifies the low-dimensional structure of the underlying data distribution, maps it into a low-dimensional latent space while preserving the underlying geometry, and then optimally transports a reference measure to the embedded distribution.  
We prove three key properties of our method: 1) The encoder preserves the geometry of the underlying data; 2) The generator is $c$-cyclically monotone, where $c$ is an intrinsic embedding cost employed by the encoder; and 3) The discriminator's modulus of continuity improves with the geometric preservation of the data. Numerical experiments demonstrate the effectiveness of our approach in generating high-quality images and exhibiting robustness to both mode collapse and training instability.

\end{abstract}


\section{Introduction}
The fundamental task of data generation requires a good approximation of the underlying distribution of the input dataset to generate new data instances that resemble the originals. 
Prominent approaches for generative modeling use neural networks (NNs) for such approximation and include Generative Adversarial Networks (GANs)~\cite{goodfellow2014generative}, Variational Auto-Encoders (VAEs)~\cite{kingma2013auto}, encoder-based GANs~\cite{larsen2016vaegan, zhu2017unpaired, srivastava2017veegan, bao2017cvae, gao2020zero}, normalizing flows~\cite{rezende2015variational}, and diffusion models~\cite{ho2020denoising}. VAE-based and GAN-based methods 
aim to first generate samples in a latent space, whose dimension is significantly lower than that of the input space, and then map these samples to the input space. The use of a low-dimensional latent space can enable efficient generation of high-quality examples when the underlying data distribution can be approximated by a sufficiently smooth low-dimensional structure. This assumption is common in addressing many data science problems and referred to as the manifold hypothesis~\cite{fefferman2016testing,pope2021intrinsic}.

Despite the widespread popularity of GAN-based and VAE-based methods, they suffer from significant drawbacks. A primary challenge is the substantial sensitivity of the generated map to the initializations of NN  parameters~\cite{bojanowski2018optimizing} and choice of NN  architecutres~\cite{kurach2019large}, which often impedes effective model training. Another notable limitation of GAN-based methods is mode collapse~\cite{pmlr-v70-arora17a,arjovsky2017towards}, where they fail to capture all modes in the underlying data distribution, resulting in generated examples that only represent a limited portion of the dataset. Despite many attempts to address these issues~\cite{metz2016unrolled, arjovsky2017wasserstein, srivastava2017veegan, gulrajani2017improved, wu2018wasserstein, liu2019wasserstein2, khrulkov2021functional, birrell2022structure, luo2023stabilizing}, these challenges persist.

To address these challenges, we introduce a novel encoder-based GAN method with intriguing theoretical properties. Our approach is distinguished from many others by carefully designing the encoder and generator with the following theoretically-guaranteed properties. 
First, we guarantee that the encoder preserves the geometry of the underlying data. This property alleviates mode collapse. To enforce this property, we introduce a novel regularization cost derived from the Gromov-Wasserstein (GW) distance~\cite{memoli2007use}, which is a variant of the OT distance for quantifying structural disparities between distributions in different metric spaces. 
The second property is the $c$-cyclical monotonicity of the generator $G$, where $c$ is a cost function defined with respect to the encoder. It is rooted in optimal transport (OT) theory and leads to enhanced training stability for different choices of NN architectures and parameters. Moreover, we can quantify the modulus of continuity of the discriminator and establish its improvement with the geometric preservation of the underlying data by the encoder. This leads to more efficient choices for the NN architecture of the discriminator, requiring fewer parameters.

\subsection{Related Works}\label{sec:related}

We review works related to our study and highlight the  advancements and limitations of these methods to situate our contributions within the existing literature.

\paragraph{Encoder-based GAN}
The integration of GANs with encoders has been explored in various models \cite{larsen2016vaegan, zhu2017unpaired, srivastava2017veegan, bao2017cvae, gao2020zero}. For example, VAEGAN \cite{larsen2016vaegan} uses an encoder to minimize the evidence lower bound loss, while VEEGAN \cite{srivastava2017veegan} uses an encoder to minimize the cross-entropy loss. Encoder-based GAN methods aim to improve mode collapse issues present in GAN-based approaches and enhance image quality over VAE-based ones. However, they are sensitive to the neural network architectures of the encoder, leading to training instability.

\paragraph{Local isometry encoders}
Recent studies have introduced encoders that locally preserve data structures: \cite{kato2020rate} proposed a rate-distortion optimization guided encoder and showed how it enabled local-isometric data embedding; \cite{gropp2020isometric} proposed a locally isometric decoder by introducing a special loss function and forming the encoder as the pseudo-inverse of the decoder; and \cite{lee2022regularized} considered a family of coordinate-invariant regularization terms to measure how closely the decoder approximates a scaled isometry. However, we are unaware of any attempts for global geometric preservation by the encoder, rather than local.

\paragraph{Generative models via OT}  
The OT distance has been extensively incorporated in generative modeling. Various investigations \cite{arjovsky2017wasserstein, gulrajani2017improved, wu2018wasserstein} have integrated the \(W_1\) distance into GAN-based methods, while others \cite{makkuva2020optimal, taghvaei20192, liu2019wasserstein2, lei2019geometric, korotin2021wasserstein, rout2022generative} have focused on the \(W_2\) distance. Our approach uses the \(W_2\) distance but computes it in the latent space. While prior works, such as \cite{korotin2021wasserstein}, have also used the Wasserstein distance in the latent space, our method distinguishes itself by employing a specialized encoder that preserves the geometric structure of the data distribution.

\paragraph{GW distance}
This distance has been widely employed to quantify structural differences across various distributions in numerous contexts \cite{memoli2011gromov, peyre2016gromov, alvarez2018gromov, xu2019scalable, xu2019gromov, li2022gromov}. It has also been integrated into specific NN architectures, such as transformers~\cite{huang2022improving}. Despite its versatility, the GW distance remains relatively underutilized in generative modeling, with only a few generative models incorporating it \cite{bunne2019learning, titouan2019sliced, nakagawa2023gromovwasserstein}. Existing methods often suffer from computational inefficiencies and may produce suboptimal results due to their non-convex formulations.

\subsection{Contribution}

We summarize the main contributions of our work.

\begin{enumerate}
\item We derive a novel generative algorithm that can be implemented using an encoder-based GAN framework. For this purpose, we introduce a novel cost derived from the GW distance and a novel OT-based framework.
\item We establish the following guarantees: a) The encoder preserves the underlying data geometry; b) The generator is $c$-cyclically monotone; and c) The discriminator's modulus of continuity improves with the geometric preservation of the data.
\item Numerical experiments on CIFAR10 and Tiny ImageNet as well as on a synthetic setting demonstrate generated objects of the highest quality with the best training and mode-coverage stability among GAN-based methods.  
\end{enumerate}

\subsection{Structure of the Rest of the paper}
\Cref{sec:intro} reviews the  mathematical framework of our proposed method and motivates its choices. 
\Cref{sec:ot-gw} introduces our novel cost function and demonstrates its ability to enforce a geometry-preserving encoder. \Cref{sec:gen} delves into additional mathematical details, provides interpretation of our OT-inspired, encoder-based GAN framework, and establishes theoretical properties of the generator and discriminator. \Cref{sec:alg} presents the main algorithm and \Cref{sec:exp} compares this algorithm  
 with other GAN-based methods (including encoder-based GANs) on both artificial and real datasets. Finally, \Cref{sec:conclusion} concludes this work.

\section{Motivating Our Method}\label{sec:intro}

Our approach involves two key components: a specialized encoder for mitigating mode collapse and the use of optimal transport costs to enforce a $c$-cyclically monotone generator and stabilize training. In order to explain the mathematical innovation of our work, we first briefly review in \Cref{sec:review-methods} the mathematical ideas behind GAN and encoder-based GAN. \Cref{sec:address_mode_collapse} 
and \Cref{sec:address_train_instable} 
explain the mathematical ideas of each new  component, while motivating them with the problems they aim to solve.

\paragraph{Notation and Conventions} Throughout the paper we 
assume a compact $d$-dimensional manifold $\mathcal{M}$ in $\mathbb{R}^D$ and a probability measure $\mu$ supported on  $\mathcal{M}$, that is, $\mu \in \mathcal{P}(\mathcal{M})$. We also assume a latent space $Y \subset \mathbb{R}^d$, where $d \ll D$ and a latent distribution $\nu\in \mathcal{P}(Y)$, where commonly $Y = \mathbb{R}^d$ and $\nu$ is the standard normal distribution whose covariance is the identity matrix, that is $\nu = N(0, \boldsymbol{I})$. 
Ideally, the aim is to compute a generator map $G:Y\rightarrow \mathcal{M}$ that pushes forward the latent distribution $\nu$ to the data distribution $\mu$, i.e., $G_\#\nu = \mu$, where the pushforward measure $G_\#\nu$ is defined by $G_\# \nu(B) = \nu(G^{-1}(B))$ for all $B \subset Y$. In practice, one only has a finite sample from $\mu$, but given the effective generalization of neural networks and the large number of samples, it is common to mathematically address this continuous setting. 

\subsection{Review of the Mathematical Frameworks of GANs and Encoder-based GANs}\label{sec:review-methods}

\paragraph{GAN}
The GAN objective is formulated as a minimax problem:
\[
\min_{G} \max_{\psi} L_{\psi}(G_\# \nu, \mu),
\]
where $L_\psi$ is a cost function involving a discriminator $\psi$  measuring the discrepancy between $G_\# \nu$ and $\mu$. The optimal generator, $G^*$, of this minimax problem satisfies $G^*_\#\nu = \mu$.

\paragraph{Encoder-based GAN}
The common framework for these methods, which deviates from our encoder-based GAN, aims to find $G$ and $\psi$ as in GAN and an encoder $T$ solving
\[
\min_{G,T} \max_{\psi} L_{\psi}(G_\# \nu, \mu) + \mathbb{E}_{x\sim \mu}\|x - G(T(x))\|^2 + C( T_\#\mu,\nu)
\]
where the second term enforces the inverse relationship between the encoder and generator, $G = T^{-1}$, and the last measures the discrepancy between $T_\#\mu$ and $\nu$ using a cost function $C$.

\subsection{Addressing Mode Collapse via a Specialized Encoder}
\label{sec:address_mode_collapse}
Mode collapse presents a significant challenge in generation tasks, where the diversity of traits in the generated distribution does not fully capture the input data distribution. 

As it was discussed in \Cref{sec:review-methods}, the main objective of GAN is to minimize the distance between two distributions: the data distribution $\mu$ and the generated distribution $G_\#\nu$. Mode collapse is often attributed to the challenge of comparing samples from $\mu$ and $G_\# \nu$ in high-dimensional space $\mathbb{R}^D$.  
To overcome the complexity of comparing distributions in high-dimensional spaces, we propose a novel encoder-based GAN method, reducing the problem to a lower-dimensional space while preserving the underlying geometry. 
This latter idea of preserving geometry, explained below, differs from the common encoder-based GAN mechanism.

We quantify the encoder preservation of the underlying geometry by the bi-Lipschitzness of this encoder. A map $T: \mathcal{M} \rightarrow Y$ is $(\alpha^{-
1})$-bi-Lipschitz if there exists $0 < \alpha \leq 1$ such that
\begin{equation}\label{eq:bi-lip}
    \alpha \|x-x'\| \leq \|T(x) - T(x')\| \leq \frac{1}{\alpha} \|x-x'\|,\quad \forall x,x'\in\mathcal{M}.
\end{equation}  
The bi-Lipschitzness of $T$ implies the following relationship between the Wasserstein-$p$ distance, $W_p$, of $G_\# \nu$ and $\mu$,  which is defined in \Cref{sec:ot}, and the $W_p$ distance of the quantities embedded by $T$, which is proved in the appendix: 
\begin{proposition}\label{prop:upper-bound}
    For $p \geq 1$ and $0 < \alpha \leq 1$, if $T: \mathcal{M} \rightarrow Y$ is $(\alpha^{-1})$-bi-Lipschitz, then
    \begin{equation*}
        %\label{eq:upper-bound}
        \alpha W_p(T \circ G_\# \nu, T_\# \mu)
        \leq W_p(G_\# \nu, \mu) \leq \frac{1}{\alpha} W_p(T \circ G_\# \nu, T_\# \mu).
        \end{equation*}
\end{proposition}
Ideally, we say that $T$ is a geometry-preserving map if it is $(\alpha^{-
1})$-bi-Lipschitz and $\alpha$ is sufficiently close to 1. We would rather keep this terminology flexible, by not specifying how close $\alpha$ is to 1 and allowing the following two relaxations: 1) The condition in \eqref{eq:bi-lip} can hold for $x$ and $x'$ in a sufficiently large subset of $\cal M$; 2) The same condition only holds for sufficiently well-separated data points $x$ and $x'$.
In view of Proposition \ref{prop:upper-bound}, a geometry-preserving encoder $T$ effectively reduces the problem from a high-dimensional to a lower-dimensional setting. 

For general distributions $\mu$ and $\nu$, a map $T$ such that $T_\#\mu=\nu$ will typically not be geometry-preserving.  
This limitation is intrinsic in earlier encoder-based GAN approaches, which aim to establish a pushforward map from the data distribution to the Gaussian distribution, and this map cannot generally preserve the geometry of the data distribution. Nevertheless, our method does not enforce the pushforward constraint and aims to obtain instead a geometry-preserving map. However, it seems impossible to directly enforce an encoder to be   
$\alpha^{-1}$-bi-Lipschitz with $\alpha$ close to 1 by an algorithm. We thus introduce a newly-proposed embedding cost 
for $T$ given $\mu$
derived from the Gromov-Wasserstein distance and show that if this cost is sufficiently small then $T$ is  geometry-preserving (see \Cref{sec:ot-gw}). In practice, the algorithm uses this new cost as a regularization term to make sure that it is sufficiently small (see \Cref{sec:alg}).

The generator $G$ no longer adheres to the inverse relationship $G = T^{-1}$ as in encoder-based GAN. That is, the map $R := T \circ G$ is different from the identity. Ideally, we aim to define $R$ as the optimal transport map between $\nu$ and $T_\#\mu$, i.e., 
\begin{align}
\label{eq:def_R}
W_2^2(\nu, T_\#\mu) = \min_{\substack{R:\\  R_\#\nu = T_\#\mu}} \int_Y \frac{\|R(y)-y\|^2}{2} d\nu(y) \equiv
\min_{\substack{G:\\  (T\circ G)_\#\nu = T_\#\mu}} \int_Y \frac{\|T\circ G(y)-y\|^2}{2} d\nu(y).
\end{align}
Since the pushforward constraint $R_\#\nu = T_\#\mu$ is hard to implement, we use a common minimax formulation involving a discriminator. \Cref{fig:three-maps} illustrates the three maps $G$, $T$, and $R$.

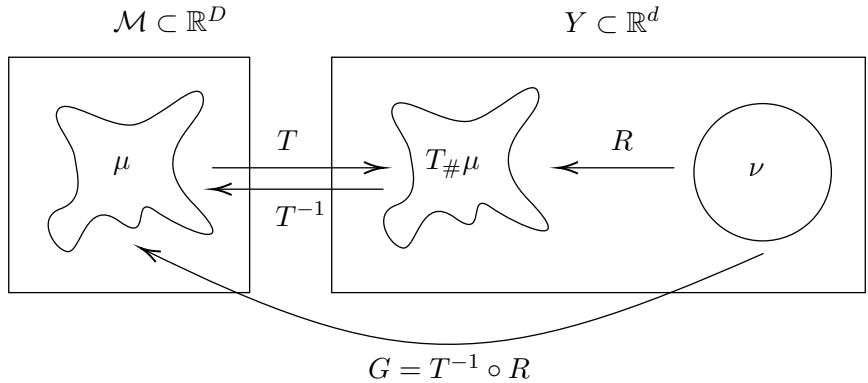
\begin{figure}[h!]
  \centering

  \tikzset{every picture/.style={line width=0.5pt}} %set default line width to 0.75pt        

  \begin{tikzpicture}[x=0.75pt,y=0.75pt,yscale=-0.9,xscale=0.9]

%Shape: Rectangle [id:dp5080997679786922] 
\draw   (81,70) -- (216,70) -- (216,202) -- (81,202) -- cycle ;
%Shape: Rectangle [id:dp6734460960828967] 
\draw   (262,70) -- (561,70) -- (561,202) -- (262,202) -- cycle ;
%Shape: Polygon Curved [id:ds5817836938693778] 
\draw   (109,99) .. controls (114.42,92.75) and (135.57,109.09) .. (150,108) .. controls (164.43,106.91) and (188,79) .. (190,93) .. controls (192,107) and (155,122) .. (184.31,152.05) .. controls (213.63,182.1) and (177,166) .. (163,156) .. controls (149,146) and (162,176) .. (143,162) .. controls (124,148) and (121.52,179.84) .. (114,179) .. controls (106.48,178.16) and (96,160) .. (110,155) .. controls (124,150) and (119.4,135.46) .. (118,126) .. controls (116.6,116.54) and (103.58,105.25) .. (109,99) -- cycle ;
%Shape: Circle [id:dp03208859772125805] 
\draw   (465,134.5) .. controls (465,113.24) and (482.24,96) .. (503.5,96) .. controls (524.76,96) and (542,113.24) .. (542,134.5) .. controls (542,155.76) and (524.76,173) .. (503.5,173) .. controls (482.24,173) and (465,155.76) .. (465,134.5) -- cycle ;
%Straight Lines [id:da8932237275780418] 
\draw    (454,132) -- (391,132) ;
\draw [shift={(389,132)}, rotate = 360] [color={rgb, 255:red, 0; green, 0; blue, 0 }  ][line width=0.75]    (10.93,-3.29) .. controls (6.95,-1.4) and (3.31,-0.3) .. (0,0) .. controls (3.31,0.3) and (6.95,1.4) .. (10.93,3.29)   ;
%Straight Lines [id:da5985604077704514] 
\draw    (195,132) -- (289,132) ;
\draw [shift={(291,132)}, rotate = 180] [color={rgb, 255:red, 0; green, 0; blue, 0 }  ][line width=0.75]    (10.93,-3.29) .. controls (6.95,-1.4) and (3.31,-0.3) .. (0,0) .. controls (3.31,0.3) and (6.95,1.4) .. (10.93,3.29)   ;
%Shape: Polygon Curved [id:ds48646588388978085] 
\draw   (297,97) .. controls (302.42,90.75) and (323.57,107.09) .. (338,106) .. controls (352.43,104.91) and (376,77) .. (378,91) .. controls (380,105) and (343,120) .. (372.31,150.05) .. controls (401.63,180.1) and (365,164) .. (351,154) .. controls (337,144) and (350,174) .. (331,160) .. controls (312,146) and (309.52,177.84) .. (302,177) .. controls (294.48,176.16) and (284,158) .. (298,153) .. controls (312,148) and (307.4,133.46) .. (306,124) .. controls (304.6,114.54) and (291.58,103.25) .. (297,97) -- cycle ;
%Straight Lines [id:da33552223750379306] 
\draw    (197,144) -- (291,144) ;
\draw [shift={(195,144)}, rotate = 0] [color={rgb, 255:red, 0; green, 0; blue, 0 }  ][line width=0.75]    (10.93,-3.29) .. controls (6.95,-1.4) and (3.31,-0.3) .. (0,0) .. controls (3.31,0.3) and (6.95,1.4) .. (10.93,3.29)   ;
%Curve Lines [id:da28886797599314407] 
\draw    (504,180) .. controls (351,248) and (300,249) .. (155,177) ;
\draw [shift={(155,177)}, rotate = 26.41] [color={rgb, 255:red, 0; green, 0; blue, 0 }  ][line width=0.75]    (10.93,-3.29) .. controls (6.95,-1.4) and (3.31,-0.3) .. (0,0) .. controls (3.31,0.3) and (6.95,1.4) .. (10.93,3.29)   ;

% Text Node
\draw (138,40) node [anchor=north west][inner sep=0.75pt]    {${\mathcal{M}}\subset \mathbb{R}^D$};
% Text Node
\draw (391,40) node [anchor=north west][inner sep=0.75pt]    {${Y}\subset \mathbb{R}^d$};
% Text Node
\draw (138,124) node [anchor=north west][inner sep=0.75pt]    {$\mu $};
% Text Node
\draw (313,120) node [anchor=north west][inner sep=0.75pt]    {$T_{\#} \mu $};
% Text Node
\draw (494,128) node [anchor=north west][inner sep=0.75pt]    {$\nu $};
% Text Node
\draw (230,109.4) node [anchor=north west][inner sep=0.75pt]    {$T$};
% Text Node
\draw (417,109.4) node [anchor=north west][inner sep=0.75pt]    {$R$};
% Text Node
\draw (229,149.4) node [anchor=north west][inner sep=0.75pt]    {$T^{-1}$};
% Text Node
\draw (281,235) node [anchor=north west][inner sep=0.75pt]    {$G=T^{-1} \circ R$};

\end{tikzpicture}
\caption{Illustration of our method for generation of samples in $\mathcal{M}$ with a latent space $Y$ and a geometry-preserving map $T$. }
\label{fig:three-maps}
\end{figure}

We note that in view of \Cref{prop:upper-bound} 
$W_2(R_\# \nu, T_\# \mu)$ controls $W_2(G_\# \nu, \mu)$ as follows: 
\begin{align}\label{eq:new-upper-bound}
    {\alpha} W_2(R_\# \nu, T_\# \mu) \leq W_2(G_\# \nu, \mu) \leq \frac{1}{\alpha} W_2(R_\# \nu, T_\# \mu).
\end{align}
That is, instead of the comparing via $G$ between $\nu$ and $\mu$, who lie in different spaces, we focus on the comparison via $R$ between $\nu$ and $T_\# \mu$, which both lie in $\mathbb{R}^d$.

\subsection{Addressing training instability by $c$-cyclical monotonicity} 
\label{sec:address_train_instable}
We further elaborate on our approach to computing the generator such that $G = T^{-1} \circ R$ using optimal transport. In particular, we discuss the 
$c$-cyclical monotonicity of $G$. 
Additionally, we highlight the benefits of this approach in addressing training instability.

We believe that training instability in GANs, VAEs and encoder-based GANs results from the non-uniqueness of the generator, satisfying non-trivial constraints. Moreover, some of the possible solutions for the generator are highly non-regular. 

To clarify this claim in a simplistic setting, %let us assume that one aims to design a generator satisfying the pushforward constraint $G_\# \nu = \mu$. We assume further that 
where both $\mu$ and $\nu$ are uniform distributions on $[0,1]$. Let's examine $G_0(x) = x$ and
\begin{align*}
    G_k(x) = 2k \left| x - {(2i+1)}/{(2k)} \right|, & & i/k \leq x \leq (i+1)/k, \  i=0,\cdots,k-1 \text{ and }  k \in \mathbb{N}. 
    \end{align*}
\begin{figure}[h!]
    \centering
\tikzset{every picture/.style={line width=0.5pt}} %set default line width to 0.75pt        

\begin{tikzpicture}[x=0.75pt,y=0.75pt,yscale=-0.6,xscale=0.6]

%Shape: Rectangle [id:dp6067590427258017] 
\draw   (313,40) -- (513,40) -- (513,240) -- (313,240) -- cycle ;
%Straight Lines [id:da6975836031011283] 
\draw    (313,40) -- (363,240) ;
%Straight Lines [id:da32793695774150067] 
\draw    (413,40) -- (363,240) ;
%Straight Lines [id:da4271430889753782] 
\draw    (413,40) -- (463,240) ;
%Straight Lines [id:da2581472957478167] 
\draw    (513,40) -- (463,240) ;
%Shape: Rectangle [id:dp5063750880596649] 
\draw   (33,40) -- (233,40) -- (233,240) -- (33,240) -- cycle ;
%Straight Lines [id:da6878933501148101] 
\draw    (233,40) -- (33,240) ;

% Text Node
\draw (315,243.4) node [anchor=north west][inner sep=0.75pt]    {$0$};
% Text Node
\draw (504,243.4) node [anchor=north west][inner sep=0.75pt]    {$1$};
% Text Node
\draw (291,42.4) node [anchor=north west][inner sep=0.75pt]    {$1$};
% Text Node
\draw (291,223.4) node [anchor=north west][inner sep=0.75pt]    {$0$};
% Text Node
\draw (35,243.4) node [anchor=north west][inner sep=0.75pt]    {$0$};
% Text Node
\draw (224,243.4) node [anchor=north west][inner sep=0.75pt]    {$1$};
% Text Node
\draw (11,42.4) node [anchor=north west][inner sep=0.75pt]    {$1$};
% Text Node
\draw (11,223.4) node [anchor=north west][inner sep=0.75pt]    {$0$};
\end{tikzpicture}
    \caption{Demonstration of the graphs of the functions $G_0$ (left) and $G_2$ (right). For any $k \in \{0\} \cup \mathbb{N}$ and for $\mu$ and $\nu$ uniform distributions on $[0,1]$, $G_k\#\nu = \mu$.}
    \label{fig:G_k}
\end{figure}
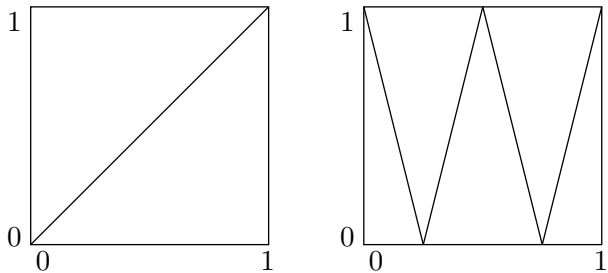

 \Cref{fig:G_k}
    demonstrates plots of $G_0$ and $G_2$.
We note that any $G_k$, for $k \geq 0$, satisfies the pushforward relationship, $G_{k\#} \nu = \mu$. However,  for $k \geq 1$, $G_k$ exhibits spikes, whose numbers increase with $k$. For large $k$, they induce highly irregular  generators. To tackle this issue in this example, one may enforce monotonicity  of the generator. As mentioned in \Cref{sec:related}, some previous works aimed to form some type of monotone generators (since $X \neq Y$  a non-standard notion of monotonicity needs to be applied), but their improvement of training stability is not significant enough. Hence, we substantially diverge from previous methods.

%In the previous subsection, we defined $G = T^{-1} \circ R$, where $R$ is an optimal transport map defined in \eqref{eq:def_R}.  
We show in \Cref{sec:gen} that the generator $G$ of our method is $c$-cyclically monotone, where $c\equiv c(x,y) = \|T(x) - y\|^2/2$ and is associated with the encoder $T$. The definition of $c$-cyclical monotonicity is the following one:
\begin{definition}[$c$-cyclical monotonicity \cite{santambrogio2015optimal}]\label{def:c-cm}
For $c:X \times Y \rightarrow \mathbb{R}$, a set $\Gamma \subset X \times Y$ is $c$-cyclically monotone ($c$-CM) if for every $k\in\mathbb{N}$, every permutation $\sigma$, and every finite family of points $(x_1,y_1),\cdots,(x_k,y_k)\in\Gamma$, 
\begin{align*}
  \sum^k_{i=1} c(x_i,y_i) \leq \sum^k_{i=1} c(x_{\sigma(i)},y_i).
\end{align*}
The map $G: X\rightarrow Y$ is a $c$-CM map if the set $\Gamma = \{(x,G(x)): x \in X\}$ is $c$-CM.
\end{definition}

Restricting the generator $G$ to be a $c$-CM map makes it more regular and well-behaved. This imposition of monotonicity significantly aids in stabilizing the optimization process, resulting in robustness against variations in NN  architectures and parameter initializations.

\section{Geometry-Preserving Maps via the Gromov-Monge Cost}\label{sec:ot-gw}

In this section, we introduce an approach to discover a geometry-preserving map that satisfies~\eqref{eq:bi-lip} using the GW cost~\cite{memoli2007use}.
The GW cost  is a variant of the OT cost which applies to two measures from heterogeneous metric spaces. Given two probability measures $\mu$ and $\nu$ defined on the metric spaces $X$ and $Y$, respectively, and cost functions $c_X: \mathcal{M}^2 \to \R$ and $c_Y: Y^2 \to \R$, the GW cost, $\operatorname{GW}(\mu, \nu)
        \equiv \operatorname{GW}_{c_X,c_Y}(\mu, \nu)$, is defined by
    \begin{align}\label{eq:GW-dis}
      \operatorname{GW}(\mu, \nu):=
        &\inf_{\pi \in \Pi(\mu,\nu) } \mathop{\Exp}_{((x,y),(x',y'))\sim \pi^2}\left[\left\vert c_{X}(x, x')- c_{Y}(y,y') \right\vert^2 \right]
    \end{align}
where $\Pi(\mu,\nu)$ is the following set of transport plans with respective marginals $\mu$ and $\nu$:
\begin{align*}
\Pi(\mu,\nu) := \big\{ \pi \in \mathcal{P}({X} \times {Y}) : \pi(A\times {Y})=\mu(A),
\pi({X}\times B) = \nu(B)
\ \ \forall \, A\subset {X}, B\subset {Y}\big\}.
\end{align*}
The GW cost thus finds the coupling between the two metric spaces that minimizes the cost of matching pairs of points from one space to the other. 

An alternative {Gromov-Monge (GM)} cost~\cite{memoli2022distance, dumont2022existence} uses 
a transport map $T:\mathcal{M} \rightarrow Y$ instead of the GW cost as follows:
\begin{align}\label{eq:GW-map}
        \operatorname{GM}(\mu, \nu) := 
        &\inf_{ T_\#\mu = \nu } \mathop{\Exp}_{(x,x')\sim \mu^2}\left[|c_{X}(x, x') - c_{Y}(T(x),T(x'))|^2\right].
\end{align}
In general, $\operatorname{GW}(\mu, \nu) \leq \operatorname{GM}(\mu, \nu)$; however, the equality can be attained under some conditions which are shown in~\cite{dumont2022existence}.
We use the GM cost due to the explicit use of $T$. Specifically, we aim to find a mapping $T:\mathcal{M} \rightarrow Y$, where $\mathcal{M}$ represents the support of $\mu$ that solves 
    \begin{equation}\label{eq:gw-t-project}
    \min_{\substack{T:\mathcal{M}\rightarrow Y }} \operatorname{GM}(\mu, T_\#\mu).
    \end{equation}
This gives rise to an unconstrained nonconvex optimization problem. Its global minimizer is attained when $c_X(x,x') = c_Y(T(x),T(x'))$ for all $x,x' \in \mathcal{M}$, i.e., when $T$ is an isometry.

We rewrite $\operatorname{GM}( T, \mu) \equiv \operatorname{GM}(\mu, T_\#\mu)$ to emphasize the roles of $T$ and $\mu$. We 
refer to $\operatorname{GM}( T, \mu)$ as the Gromov-Monge Embedding (GME) cost.

We show that if $\operatorname{GM}( T, \mu)$ is sufficiently small then $T$ is geometry-preserving.  This will allow us to enforce a geometry-preserving encoder in our implementation. 
The theorem employs the following notation and assumptions, which we summarize here and do not repeat in its formulation: Let $\mathcal{M}$ be a $d$-dimensional compact submanifold of  $\mathbb{R}^D$ and $\mu$ be a probability distribution on  $\mathcal{M}$. We assume the cost functions 
\begin{equation}
\label{eq:cost_function_log}
c_X(x,x') = \log(1+\|x-x'\|^2) \ \text{ and } \ c_Y(y,y')=\log(1+\|y-y'\|^2).    
\end{equation}
These functions are more common in practice and often yield tighter estimates. 

\begin{theorem} \label{thm:local-minimizer}
    Fix $\epsilon>0$ and  $0<\alpha<1$. If $T:\mathcal{M}\rightarrow Y$ satisfies $\operatorname{GM}(T,\mu) < \epsilon$, then the set 
    \begin{align*}
        K &= \left\{ (x,x') \in \mathcal{M}^2: \ \alpha \leq \frac{\|T(x) - T(x')\|^2+1}{\|x-x'\|^2 + 1} \leq \frac{1}{\alpha} \right\}
    \end{align*}
    satisfies
    \begin{equation*}
        \mu^2[K] = \int_{K} d\mu d\mu> 1 - \frac{\epsilon}{(\log \alpha)^2}.
    \end{equation*}
    Moreover, if $0< \gamma < 1$ and $x$, $x' \in K$ are sufficiently separated as follows: $\|x-x'\|^2\geq \frac{1-\alpha}{\alpha \gamma}$, then $T$ satisfies 
    the following bi-Lipschitz condition for such points: 
    \begin{align}
    \label{eq:modified_bi_lip}
        \alpha(1-\gamma) \|x-x'\|^2 \leq \|T(x) - T(x')\|^2 \leq \left(\frac{1}{\alpha} + \gamma \right)\|x-x'\|^2.
    \end{align}

\end{theorem}
\begin{proof}
    Define the sets
    \begin{align*}
        B &= \left\{ (x,x') \in \mathcal{M}^2: \ \frac{\|T(x) - T(x')\|^2+1}{\|x-x'\|^2 + 1} \geq \frac{1}{\alpha} \right\}
        ,\\
        Q &= \left\{ (x,x') \in \mathcal{M}^2: \ \frac{\|T(x) - T(x')\|^2+1}{\|x-x'\|^2 + 1} \leq \alpha \right\}.
    \end{align*}
    Using the definitions of $\operatorname{GM}(T,\mu)$, $B$, the following equivalent definition of $Q$: $$Q \equiv \left\{ (x,x') \in \mathcal{M}^2: \ \frac{\|x-x'\|^2+1}{\|T(x) - T(x')\|^2 + 1} \geq 1/\alpha \right\},$$
    and the definition of $K$, we obtain 
    \begin{align*}
        \operatorname{GM}(T,\mu) 
        &= \int_{\mathcal{M}^2} \left(\log\left(\frac{\|T(x)-T(x')\|^2+1}{1+\|x-x'\|^2}\right)\right)^2 d\mu d\mu\\
        &\geq   (\log(1/\alpha))^2 \int_{\mathcal{M}^2 \cap B} d\mu d\mu + (\log(1/\alpha))^2 \int_{\mathcal{M}^2 \cap Q} d\mu d\mu\\
        &\geq  (\log \alpha)^2 \left(\int_{\mathcal{M}^2 \cap B} d\mu d\mu +  \int_{\mathcal{M}^2 \cap Q} d\mu d\mu\right)\\
        &\geq  (\log \alpha)^2 \left(1 - \int_{K} d\mu d\mu\right).
    \end{align*}
   Consequently,
    \begin{align*}
        \mu^2[K] \geq 1 -  \frac{\operatorname{GM}(T,\mu) }{(\log \alpha)^2} \geq 1 - \frac{\epsilon}{(\log \alpha)^2}.
    \end{align*}
    
    Next, we prove \eqref{eq:modified_bi_lip}. Note that the defining condition of $K$ can be expressed as 
    \begin{equation}\label{eq:K-alt}
        \alpha \|x-x'\|^2 - (1-\alpha)
        \leq \|T(x) -T(x')\|^2 
        \leq \frac{1}{\alpha}\|x-x'\|^2 + \left(    \frac{1}{\alpha}- 1 \right).   
    \end{equation}
    Fix $0<\gamma<1$ and $x,x'\in K$ satisfying $\|x-x'\|^2\geq\frac{1-\alpha}{\alpha\gamma}$. From the lower bound in~\eqref{eq:K-alt},
    \[
    \alpha \|x-x'\|^2 - (1-\alpha)
    \geq
    \alpha \|x-x'\|^2 - \alpha \gamma \|x-x'\|^2
    =
    \alpha (1 -  \gamma \|)x-x'\|^2.
    \]
    The upper bound can be shown similarly:
    \[
    \frac{1}{\alpha }\|x-x'\|^2 + \left(\frac{1}{\alpha} - 1\right)
    \leq
    \frac{1}{\alpha }\|x-x'\|^2
    +
    \gamma\|x-x'\|^2
    \leq \left(\frac{1}{\alpha} +\gamma\right) \|x-x'\|^2.
    \]
\end{proof}

\begin{remark}
   For sufficiently small $\gamma$ and $\epsilon$, $\alpha$ can be chosen close to one such that $\frac{1-\alpha}{\alpha \gamma}$ is close to zero  and both $\mu^2[K]$ and the bi-Lipschitz constant of \eqref{eq:modified_bi_lip} are close to 1.
\end{remark}

\Cref{fig:GME-cost-ex} tests whether our GME-based encoder is geometry-preserving in practice, while comparing it to the encoder of VAE, which is the same encoder of VAEGAN.   
This figure plots 
\begin{equation}
\label{eq:Lip_const_for_T}
{\|T(x)-T(x')\|}/{\|x-x'\|}    
\end{equation}
as a function of $\|x-x'\|$. 
It uses both the MNIST and CIFAR10 datasets and its latent space is $\mathbb{R}^{100}$, which is a common choice by GAN-based models for both datasets.

\begin{figure}[!h]
    \centering
    \begin{subfigure}[Ours (MNIST)]{\includegraphics[height=0.142\textheight]{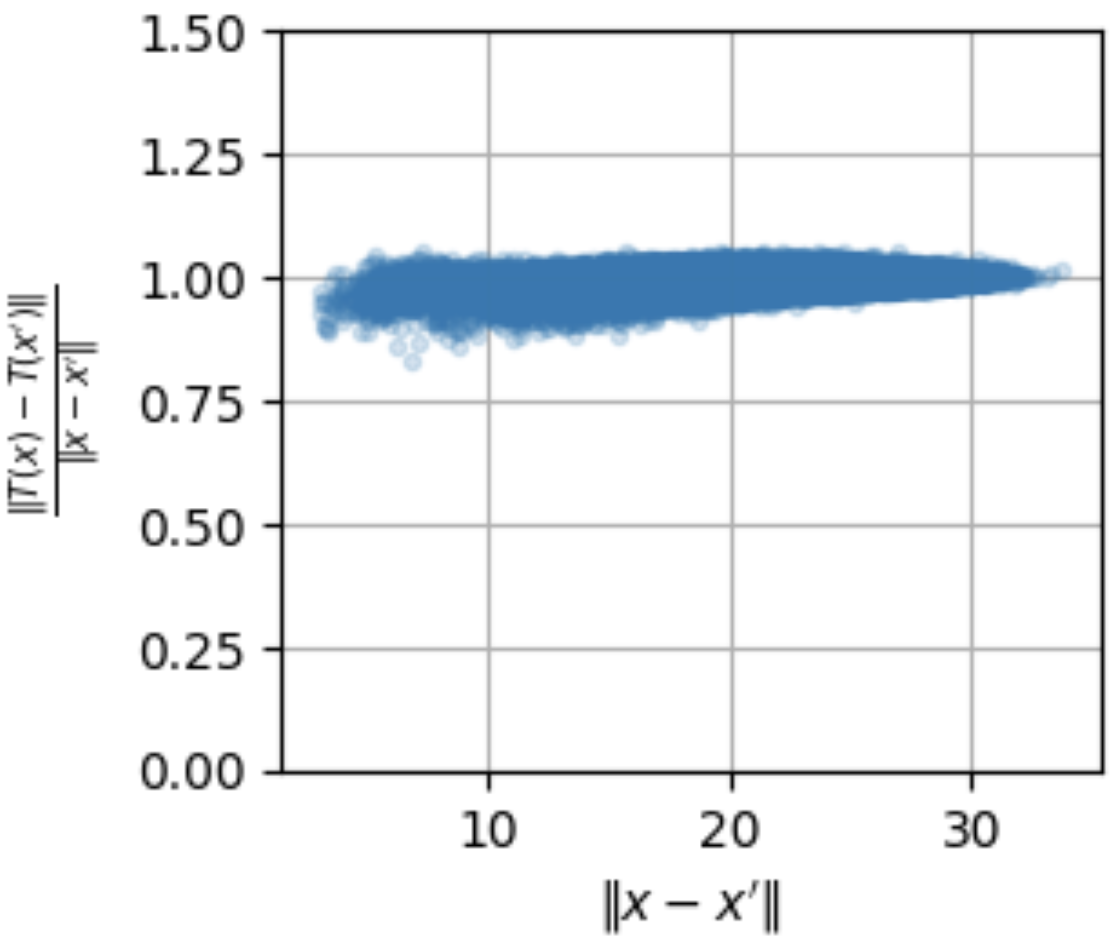}}
    \end{subfigure}
    \begin{subfigure}[VAE (MNIST)]{\includegraphics[height=0.142\textheight]{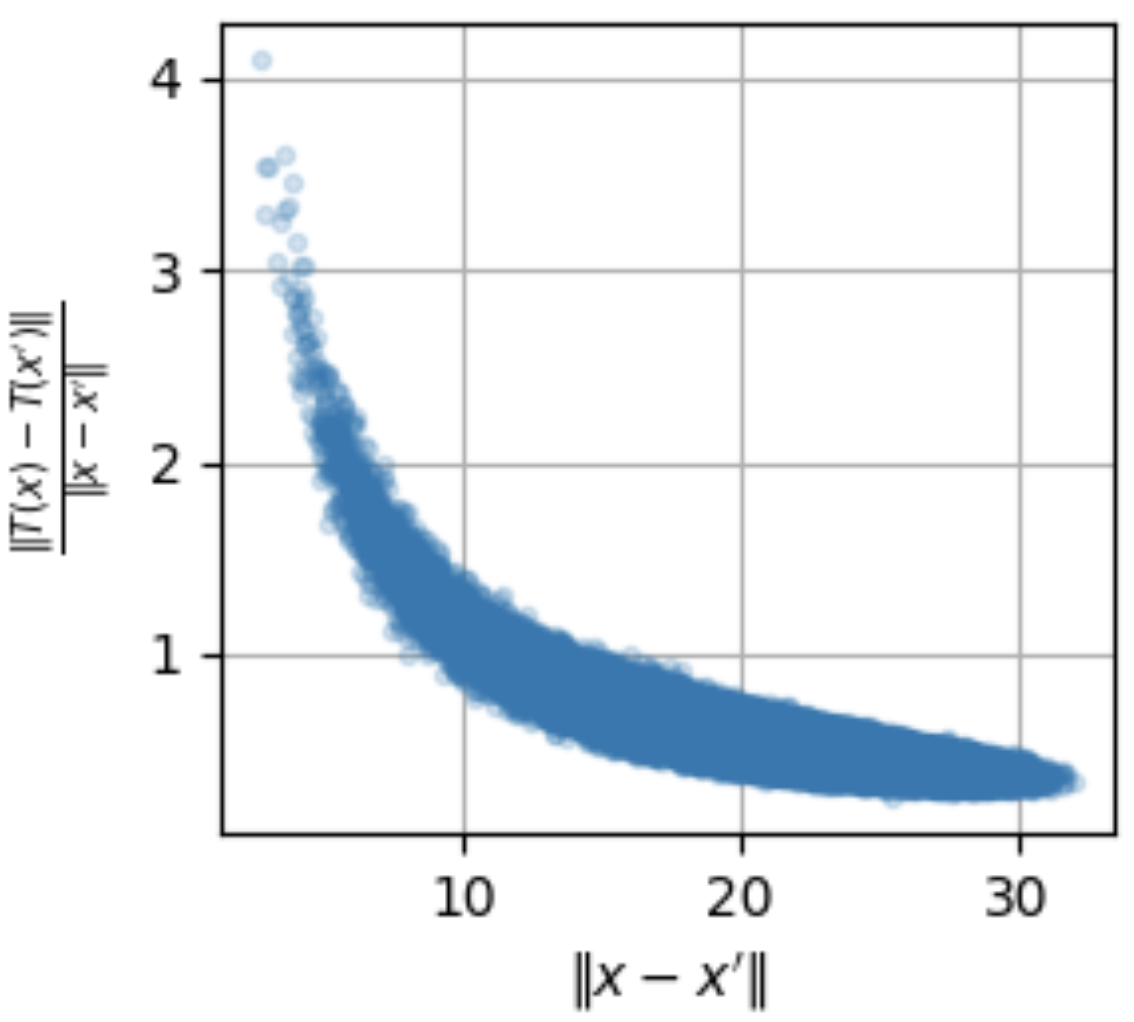}}
    \end{subfigure}
    \begin{subfigure}[Ours (CIFAR10)]{\includegraphics[height=0.142\textheight]{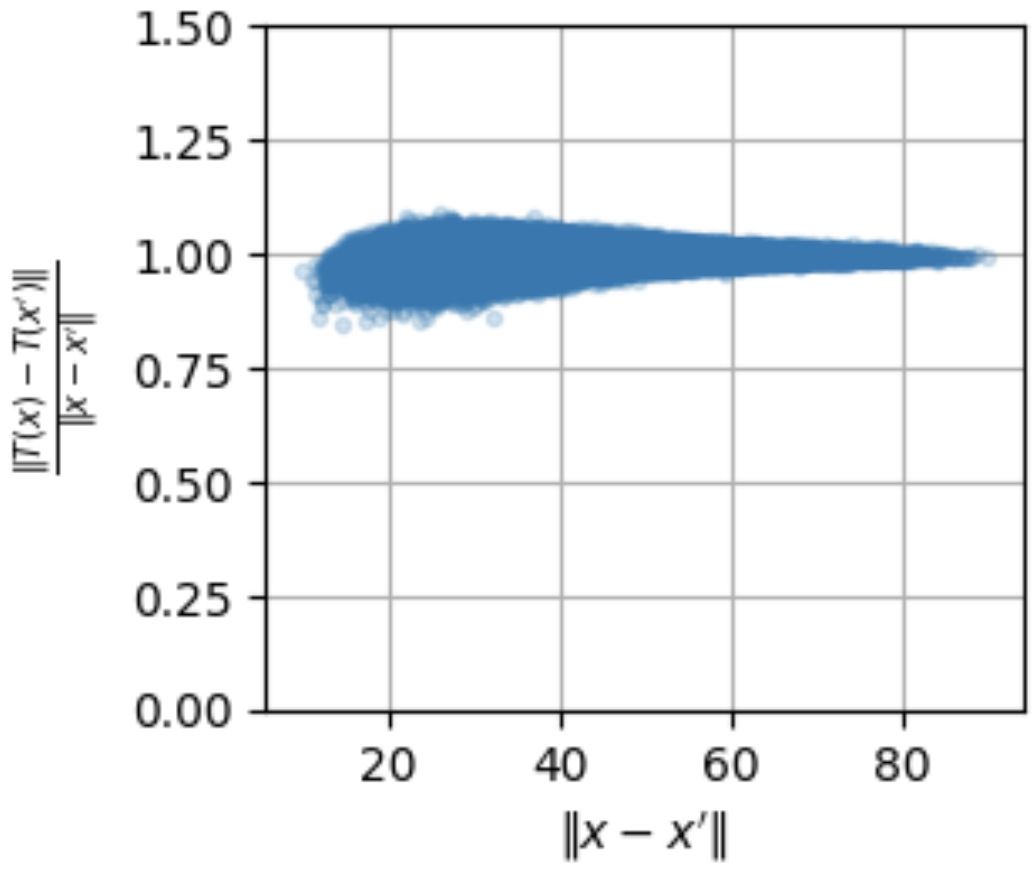}}
    \end{subfigure}
    \begin{subfigure}[VAE (CIFAR10)]{\includegraphics[height=0.142\textheight]{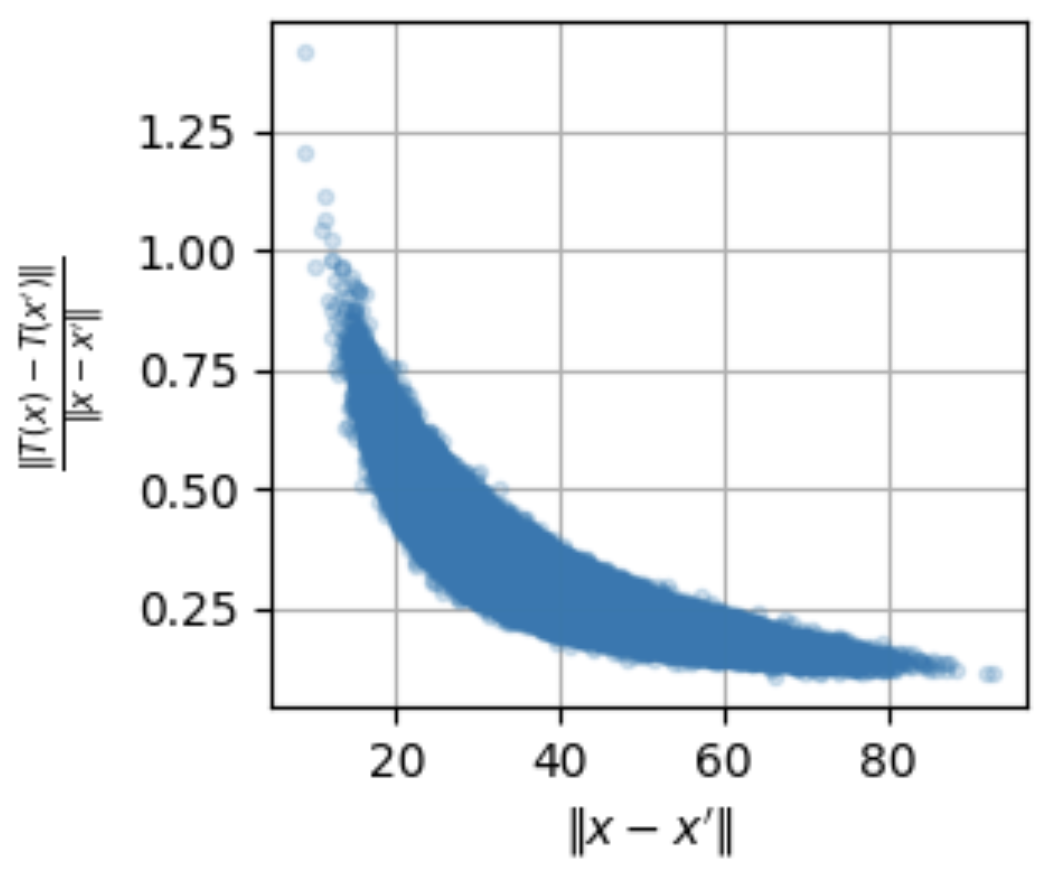}}
    \end{subfigure}
    \caption{Scatter plots depicting the ratio  $\frac{\|T(x)-T(x')\|}{\|x-x'\|}$ versus 
    $\|x-x'\|$ for encoders $T$ obtained by our GME-based method and VAE. They are applied to both the MNIST and CIFAR10 datasets. We use the commonly implemented latent space for these datasets,  $\mathbb{R}^{100}$. Clearly, our encoder is geometry-preserving, unlike the VAE encoder.
    }
    \label{fig:GME-cost-ex}
\end{figure}

We note that the ratio in \eqref{eq:Lip_const_for_T} is close to 1 for our encoder and the higher ${\|x-x'\|}$ the closer it is to 1. That is, our encoder $T$ is geometry-preserving, and this property is more emphasized for well-separated points. For VAE, this ratio significantly varies with ${\|x-x'\|}$ and is often away from 1.  That is, $T$ in VAE does not preserve the underlying geometry. 

Since we noticed in this and other experiments that $T$ is in practice bi-Lipschitz and consequently it is invertible, we assume in the next section the invertibility of $T$.

\section{Generation via Geometry-Preserving Maps}\label{sec:gen}

We extend the mathematical foundations of our generative model. 
\Cref{sec:ot} reviews some notions of optimal transport and use them to show that our proposed generator $G$ can be presented as an optimal map for a special optimal transport cost, depending on the geometry-preserving generator $T$, which satisfies the desired relationship: $G^*_\#\nu = \mu$. 
\Cref{sec:ccm-map} establishes the $c$-cyclical monotonicity of the generative map and a regularity property of the discriminator.

\subsection{Optimal Transport}\label{sec:ot}

Given a cost function $c: {\mathcal{M}} \times {Y} \rightarrow \mathbb{R}$, the corresponding OT cost between $\mu$ and $\nu$ takes the form
\begin{equation}\label{eq:ot}
\operatorname{OT}_{c}(\mu,\nu) := \inf_{\pi \in \Pi(\mu,\nu)} \mathop{\Exp}_{(x,y) \sim \pi}[c(x,y)].
\end{equation}
For ${D} = {d}$, $p \in [1, \infty)$ and $c_p(x,y)=\|x-y\|^p/p \ $, the Wasserstein-$p$ distance is defined by 
$$W_p^p(\mu,\nu) = \sqrt[p]{{\operatorname{OT}_{c_p}(\mu,\nu)}}.$$

The following lemma implies that for the cost function 
\begin{equation}
\label{eq:def_c}
c_T(x,y) = \|T(x) - y\|^2/2 \quad \forall x \in \mathcal{M} \ \text{ and } \ y \in Y,
\end{equation}
defined with respect to an invertible  encoder $T$,  the above OT cost,  $\operatorname{OT}_{c_T}(\mu,\nu)$, coincides with  $W^2_2(T_\#\mu, \nu)$, which is used in our formulation in \eqref{eq:def_R}.
\begin{lemma}\label{lem:ot-W}
    If $T: \mathcal{M} \rightarrow Y$ is invertible, $\nu$ is absolutely continuous and $c_T$ is defined in \eqref{eq:def_c}, then
    \[
        W^2_2(T_\#\mu, \nu) = \operatorname{OT}_{c_T}(\mu,\nu).
    \]
\end{lemma}
\begin{proof}
    
Let $\pi \in\Pi(\mu,\nu)$. First, we show that the coupling $\pi'$ defined by
\begin{equation}\label{eq:two-pi}
    {\pi}'(z,y) = {\pi}(T^{-1}(z),y),\quad \forall (z,y) \in Y^2
\end{equation}
satisfies
$\pi'\in\Pi(T_\#\mu,\nu)$.
    From the definition of $\pi'$, for any function $f:Y\rightarrow \mathbb{R}$,
    \begin{multline*}
        \int_Y f(z) dT_\#\mu(z) = \int_\mathcal{M} f(T(x)) d\mu(x)
        = \int_{\mathcal{M}\times Y} f(T(x)) d\pi(x,y)\\
        = \int_{Y^2} f(z) d\pi(T^{-1}(z),y)
        = \int_{Y^2} f(z) d\pi'(z,y).
    \end{multline*}
    For any function $g:Y\rightarrow \mathbb{R}$,
    \begin{align*}
        \int_Y g(y) d\nu(y) &= \int_{\mathcal{M}\times Y} g(y) d\pi(x,y) = \int_{Y^2} g(y) d\pi'(z,y).
    \end{align*}
Therefore, $\pi'\in\Pi(T_\#\mu,\nu)$. Applying~\eqref{eq:two-pi}, we conclude the lemma as follows:
\begin{align*}
    \operatorname{OT}_{c_T}(\mu,\nu) &= \inf_{\pi \in \Pi(\mu,\nu)} \int_{\mathcal{M}\times Y} c_T(x,y) d\pi(x,y)\\
    &= \inf_{\pi' \in \Pi(T_\#\mu,\nu)} \int_{Y \times Y} \frac{\|y'-y\|^2}{2} d\pi'(y',y)
    = W^2_2(T_\#\mu,\nu).
\end{align*}
\end{proof}
Lemma~\ref{lem:ot-W} is crucial for transitioning from solving distances between $T_\#\mu$ and $\nu$ to distances between $\mu$ and $\nu$. Using this formulation, the following theorem proves the existence of a minimizer for~\eqref{eq:ot} with the cost function in \eqref{eq:def_c} defined with respect to the encoder $T$. Moreover, it shows that this minimizer, an OT plan, is induced by an OT map $G$ such that $G_\#\nu=\mu$, where $G$ is the desired generator. 
\begin{theorem}\label{thm:existence}
    If $T:\mathcal{M}\rightarrow Y$ is invertible and $c(x,y)$ is defined in \eqref{eq:def_c}, then the OT$_c$ cost~\eqref{eq:ot} admits a minimizer $\pi^* \in \Pi(\mu,\nu)$ and the OT plan $\pi^*$ is induced by an OT map $G^*:Y\rightarrow \mathcal{M}$ satisfying $G^*_\#\nu = \mu$. Furthermore, $G^*$ takes the form of $G^* = T^{-1} \circ R^*$ where $R^*:Y\rightarrow Y$ is the minimizer of \eqref{eq:def_R}.
\end{theorem}
\begin{proof}
By Lemma~\ref{lem:ot-W}, $\operatorname{OT}_{c_T}(\mu,\nu) = W^2_2(T_\#\mu,\nu)$ and by Brenier's theorem~\cite{brenier1987}, $W^2_2(T_\#\mu,\nu)$ admits an optimal transport map $R^*$ satisfying $R^*_\#\nu = T_\#\mu$ (i.e., $R^*$ is the minimizer of \eqref{eq:def_R}). We define $G^* := T^{-1} \circ R^*$. Using the definition of pushforward measures and the constraint $R^*_\#\nu = T_\#\mu$, we obtain for any Borel measurable set $A \subset X$, 
    \begin{align*}
        G^*_\#\nu(A) &= (T^{-1}\circ R^*)_\# \nu(A)\\
        &= \nu\big( (T^{-1}\circ R^*)^{-1}(A) \big) = \nu\big( (R^*)^{-1} \circ T(A) \big)\\
        &= R^*_\# \nu\big( T(A) \big) = T_\# \mu\big( T(A) \big)\ = \mu(A).
    \end{align*}
    Therefore, the  map $G^*$ satisfies $G^*_\# \nu = \mu$. 
    
    We conclude the proof by verifying that $G^*$ is an optimal transport map for $\operatorname{OT}_{c_T}(\mu,\nu)$:
    \begin{align*}
        &\operatorname{OT}_{c_T}(\mu,\nu) = \\
        & W^2_2(T_\#\mu,\nu)
        = \int_Y \|R^*(y) - y\|^2 d\nu(y)
        = \int_Y \|T\circ G^*(y) - y\|^2 d\nu(y)
        = \int_Y c_T(G^*(y),y) d\nu(y).
    \end{align*}
\end{proof}

By \Cref{thm:existence}, the OT problem~\eqref{eq:ot} with $c_T$ chosen in \eqref{eq:def_c} and $T$ invertible, can be written as a minimization with respect to the OT map $G$:
\begin{equation}\label{eq:ot-c-G}
    \operatorname{OT}_{c_T}(\mu,\nu) = \min_{\substack{G :Y \rightarrow \mathcal{M}\\G_\#\nu=\mu}} \displaystyle \mathop{\Exp}_{y \sim \nu}[\|T\circ G(y)-y\|^2 / 2].
\end{equation}
Following a standard duality argument, we introduce a dual variable $\psi:\mathcal{M} \rightarrow \mathbb{R}$ such that
\begin{align}\label{eq:ot-minimax}
    \operatorname{OT}_{c_T}(\mu,\nu) &=\min_{\substack{G :Y \rightarrow \mathcal{M}}} \max_{\psi: \mathcal{M}\rightarrow \mathbb{R}} \displaystyle \mathop{\Exp}_{y \sim \nu}[\|T \circ G(y)-y\|^2 / 2] 
    + \mathop{\Exp}_{y \sim \nu} [\psi(G(y))] - \displaystyle \mathop{\Exp}_{x \sim \mu} [\psi(x)].
\end{align}
In the practical implementation $\psi$ serves as the discriminator.

\subsection{Regularity of the Generator and Discriminator}\label{sec:ccm-map}

We first establish the $c_T$-cyclical monotonicity ($c_T$-CM) property (see Definition~\ref{def:c-cm}) of the generative map $G$. 
\begin{theorem}[$c_T$-Cyclical Monotonicity of $G$]\label{thm:c-cm}
    If $\nu \in {\cal P}(Y)$ are absolutely continuous, $T$ is invertible, 
    $R:Y\rightarrow Y$ is the minimizer of \eqref{eq:def_R} 
    and  $G = T^{-1}\circ R$, then the set
      $\Gamma = \{ (G(y),y) \in \mathcal{M} \times Y\}$
    %\end{align*}
    is $c_T$-CM.
\end{theorem}
\begin{proof}
    We arbitrarily fix  $k\in\mathbb{N}$, a permutation $\sigma$, and a finite family of points $(x_1,y_1)$, $\cdots$, $(x_k,y_k)\in\Gamma$. Using the definition of $c$-cyclical monotonicity and the definition of the optimal transport map, we have
    \begin{align*}
      \sum^k_{i=1} c_T(x_{\sigma(i)},y_i)
      = \sum^k_{i=1} \|T(x_{\sigma(i)}) - y_i\|^2
      = \sum^k_{i=1} \|T(G(y_{\sigma(i)})) - y_i\|^2 = \sum^k_{i=1} \|R(y_{\sigma(i)}) - y_i\|^2,
    \end{align*}
    where the last equality comes from the definition of $R = T\circ G$. Since $R$ is an optimal transport map, by Brenier's theorem~\cite{brenier1987}, $R$ is $c_2$-CM map with $c_2(x,y) = \|x-y\|^2/2$. Thus, it satisfies $\sum^k_{i=1} \|R(y_{\sigma(i)}) - y_i\|^2
    \geq\sum^k_{i=1} \|R(y_{i}) - y_i\|^2$ for any perturbation $\sigma$ and $k$. Using this observation, we conclude that $G$ is $c_T$-CM as follows:
    \begin{multline*}
        \sum^k_{i=1} c_T(x_{\sigma(i)},y_i) \geq\sum^k_{i=1} \|R(y_{i}) - y_i\|^2/2
        = \sum^k_{i=1} \|T(G(y_{i})) - y_i\|^2/2\\
        = \sum^k_{i=1} \|T(x_i) - y_i\|^2/2
        = \sum^k_{i=1} c_T(x_i,y_i).
      \end{multline*}
  \end{proof}
  
Next, we establish a regularity property of the discriminator $\psi$ and then interpret it.

\begin{proposition}[Regularity of the discriminator]\label{thm:dual-property}
    If $T: \mathcal{M} \rightarrow Y$ is $(\alpha^{-1})$-bi-Lipschitz and $(f, g)$ is an optimal dual pair of the optimal transport cost $W_2(T_\#\mu, \nu)$ such that $f,g:Y\rightarrow \mathbb{R}$ and
    \begin{align*}
        W_2^2(T_\#\mu, \nu) = \int_Y f(y) \, dT_\#\mu(y) + \int_Y g(z) \, d\nu(z).
    \end{align*}
    Let $\omega: \mathbb{R}_+ \rightarrow \mathbb{R}_+$ be an increasing continuous function with $\omega(0) = 0$ such that
    \begin{align*}
        |f(y) - f(y')| \leq \omega(\|y - y'\|), \quad \forall y, y' \in Y.
    \end{align*}
    Then the optimal dual variable $\psi^*$ of the minimax problem~\eqref{eq:ot-minimax} satisfies
    \[
        |\psi^*(x) - \psi^*(x')| \leq \omega\left( \frac{\|x - x'\|}{\alpha} \right), \quad \forall x, x' \in \mathcal{M}.
    \]
    Thus, if $\alpha=1$, then $\psi^*$ shares the same modulus of continuity of $f$.
\end{proposition}

\begin{proof}
    By optimal transport theory, the cost $\operatorname{OT}_{c_T}(\mu, \nu)$  can be written as a maximization problem with respect to two dual variables $\phi: Y \rightarrow \mathbb{R}$ and $\psi: \mathcal{M} \rightarrow \mathbb{R}$ such that
    \begin{align*}
        \operatorname{OT}_{c_T}(\mu, \nu) &=
        \sup_{\phi, \psi} \left\{ \int_{\mathcal{M}} \psi(x) \, d\mu(x) + \int_Y \phi(y) \, d\nu(y) : \psi(x) + \phi(y) \leq c_T(x, y) \right\}.
    \end{align*}
    Let $(\phi^*, \psi^*)$ be an optimal dual pair satisfying
    \begin{align*}
        \operatorname{OT}_{c_T}(\mu, \nu) &=
        \int_{\mathcal{M}} \psi^*(x) \, d\mu(x) + \int_Y \phi^*(y) \, d\nu(y).
    \end{align*}
    We write $\psi^* = f \circ T$ for some function $f: Y \rightarrow \mathbb{R}$. Then, the above can be written as
    \begin{align*}
        \operatorname{OT}_{c_T}(\mu, \nu) &=
        \int_{\mathcal{M}} f(T(x)) \, d\mu(x) + \int_Y \phi^*(y) \, d\nu(y)
        =\int_{\mathcal{M}} f(y) \, dT_\#\mu(y) + \int_Y \phi^*(y) \, d\nu(y).
    \end{align*}
    Therefore, it follows that $(f,\phi^*)$ is an optimal dual pair for the dual problem of $W^2_2(T_\#\mu, \nu)$. Let $\omega$ be the modulus of continuity of $f$.
    Then,
    \begin{align*}
        |\psi(x) - \psi(x')| = |f(T(x)) - f(T(x'))| \leq \omega(\|T(x) - T(x')\|) \leq \omega\left( \frac{\|x - x'\|}{\alpha} \right).
    \end{align*}
\end{proof}

The modulus of continuity of the discriminator, as indicated in \Cref{thm:dual-property}, depends on the parameter $\alpha$. 
It is known that the modulus of continuity of $f$ is equivalent to that of the quadratic function $\|y-y'\|^2$ due to the use of the Wasserstein-2 distance \cite{santambrogio2015optimal} and furthermore, the function $y\mapsto \frac{1}{2}\|y\|^2 - f(y)$ is convex. 
As $\alpha$ approaches 1, the discriminator achieves the same regularity as the optimal dual variable $f$ of $W^2_2(T_\#\mu, \nu)$, which is the same of the quadratic function $\|y-y'\|^2$. 
The regularity of the discriminator depends on the choice of NN  architecture. When the encoder deviates from preserving the geometry, particularly as the bi-Lipschitz constant $\alpha$ approaches 0, the discriminator becomes more irregular and complex. Consequently, a neural network with more parameters is required to approximate this function accurately. Conversely, a more regular discriminator allows 
less NN parameters.

\section{Algorithm}
\label{sec:alg}

We present our actual numerical algorithm that aims to solve the minimax problem outlined in~\eqref{eq:ot-minimax}. However, 
we need to introduce additional regularization terms. First of all, in order to ensure that the GME cost is sufficiently small and thus $T$ is geometry-preserving, we add  the $\operatorname{GME}$ cost as a regularization parameter with a hyperparameter $\lambda_1$. For the discriminator we add a gradient-penalty term with a hyperparmeter $\lambda_2$. This is a common procedure in OT-based GAN methods. Lastly, we add a reconstruction term with hyperparamter $\lambda_3$, since this is a common practice in encoder-based GANs formulations. 
The common reconstruction term in encoder-based GANs is
\[
\mathbb{E}_{x\sim\mu}\|G \circ T(x) -x \|^2,
\]
which enforces the relationship $G^{-1} = T$. However, in our case $G^{-1} = R^{-1} \circ T$ and in order to form a reconstruction term, we introduce the independent function $R_{inv}$, which approximates $R^{-1}$, and the reconstruction term 
\[
\mathbb{E}_{x\sim\mu}\|G \circ R_{inv} \circ T(x) - x \|^2.
\]

Consequently, our algorithm aims to solve
\begin{align}\label{eq:total-loss-minimax}
    &\min_{\substack{T:\mathcal{M} \rightarrow Y\\G :Y \rightarrow \mathcal{M}\\R_{inv}: Y\rightarrow Y} } \max_{\psi: X\rightarrow \mathbb{R}} \mathcal{L}(G,T,\psi,R_{inv}),
\end{align}
where
\begin{multline}\label{eq:total-loss}
    \mathcal{L}(G,T,\psi,R_{inv}) := \\
    \mathop{\Exp}_{y \sim \nu}[ \|T(G(y))-y\|^2/2 ] + \lambda_1 \operatorname{GM}(T;\mu) + \mathcal{D}(\psi,  \lambda_2) + \lambda_3 \ \mathcal{C}(G,T,R_{inv})
\end{multline}
with the discriminator loss
\[
\mathcal{D}(\psi,  \lambda_2) := \mathop{\Exp}_{y \sim \nu} [\psi(G(y))] - \displaystyle \mathop{\Exp}_{x \sim \mu} [\psi(x)] + \lambda_2 \mathop{\Exp}_{x \sim \mu}[\|\nabla \psi(x)\|^2],
\]
and the reconstruction loss
\[
\mathcal{C}(G,T,R_{inv}) := \mathbb{E}_{x\sim \mu} [\|G({R}_{inv}(T(x))) - x\|^2].
\]

We construct NN functions $G_{\vtheta_1}$, $T_{\vtheta_2}$, $\psi_{\vtheta_3}$, and $(R_{inv})_{\vtheta_4}$ to approximate $G$, a generative map; $T$, a geometry-preserving encoder; $\psi$, a discriminator; and $R_{inv}$, a function to enforce the inverse relation $G \circ (R_{inv} \circ T) = \text{id}$, respectively. Here, $\vtheta_1$, $\vtheta_2$, $\vtheta_3$, and $\vtheta_4$ represent parameter vectors for the respective networks. The resulting algorithm, referred to as the Gromov-Monge Embedding GAN (GMEGAN) is summarized in \Cref{alg:GMEGAN}. While the variable $R_{inv}$ requires an additional neural network in comparison to other encoder-based GAN formulations, we emphasize that in our implementation this neural network is rather simple. Indeed, it is composed of three fully-connected layers with ReLU activation functions. 

\begin{algorithm}[!ht]
   \caption{GMEGAN solving~\eqref{eq:total-loss-minimax}}\label{alg:GMEGAN}
\begin{algorithmic}
   \STATE {\bfseries Input:} A dataset $\{x_i\}^n_{i=1}$ in $X\subset \R^D$; a known distribution $\nu \in \mathcal{P}({Y})$, where ${Y}\subset \R^d$; cost function  $\mathcal{L}$ defined in~\eqref{eq:total-loss}; 
   four NNs whose learned parameters are denoted by $\vtheta_i$, $i=1,2,3,4$; initial NN parameters $(\vtheta_i)_{0}$, $i=1,2,3,4$; regularization 
   parameters $\lambda_i$, $i=1,2,3$; learning rates   $\beta_i$,  $i=1,2,3,4$; minibatch size $m \in \mathbb{N}$
   \STATE {\bfseries Output:} Generative map $G^*_{\vtheta_1}:Y\rightarrow \mathcal{M}$ and a generated distribution $(G^*_{\vtheta_1})_\#\nu$
   \medskip
   \STATE Initialize $\vtheta_i^{(0)} \gets (\vtheta_i)_0$, \  $i=1,2,3,4$
   \FOR{$k=0,1,2,\cdots$}
   \STATE Sample $\{x_i\}^m_{i=1}\subset \{x_i\}^n_{i=1}$ and $\{y_i\}^m_{i=1}$ from $\nu$\;
   \STATE $\vtheta_1^{(k+1)} \gets \vtheta_1^{(k)} - \beta_1 \nabla_{\vtheta_1} \mathcal{L} $
   \STATE $\vtheta_2^{(k+1)} \gets \vtheta_2^{(k)} - \beta_2 \nabla_{\vtheta_2} \mathcal{L}$ 
   \STATE $\vtheta_3^{(k+1)} \gets \vtheta_3^{(k)} + \beta_3 \nabla_{\vtheta_3} \mathcal{L}$
   \STATE $\vtheta_4^{(k+1)} \gets \vtheta_4^{(k)} - \beta_4 \nabla_{\vtheta_4} \mathcal{L}$
   \ENDFOR \hfill
\end{algorithmic}
\end{algorithm}

Our proposed method shares similarities with previous encoder-based GANs, but there are significant differences. First of all, the previous ones aim to find encoders that pushes forward a data distribution to a Gaussian distribution in a latent space, potentially leading to a large upper bound in \Cref{prop:upper-bound} due to a lack of control over the parameter $\alpha$. In contrast, GMEGAN distinguishes itself by incorporating a regularizer that enforces a geometry-preserving encoder. Additionally, the loss function in GMEGAN is derived using a carefully justified OT cost, which implies $c_T$-cyclical monotonicity of the generator and improved modulus of continuity for the  discriminator.  
Another difference is that in previous encoder-based GANs, the generator $G$ is the inverse of the encoder, i.e., $G = T^{-1}$. In our case, the generator $G$ takes the form $G=T^{-1}\circ R$.

Lastly, we remark that that the term $\|T\circ G(y) - y\|^2$ in the loss function~\eqref{eq:total-loss} appears in both VEEGAN and CycleGAN~\cite{zhu2017unpaired}. However, while VEEGAN and CycleGAN introduce this loss term to enforce the inverse relationship $T = G^{-1}$, GMEGAN introduces the term to compute the OT map solving  $\operatorname{OT}_{c_T}(\mu, \nu)$.

\section{Numerical Experiments}\label{sec:exp}

We describe our comprehensive numerical experiments. 
To ensure a fair comparison, we exclusively evaluate our approach against other GAN-based algorithms (including encoder-based GANs) and keep the NN architectures identical whenever possible. 
Our baseline methods include GAN~\cite{goodfellow2014generative}, Wasserstein GAN (WGAN)~\cite{arjovsky2017wasserstein}, Wasserstein GAN with gradient penalty (WGP)~\cite{gulrajani2017improved}, Wasserstein Divergence for GANs (WDIV)~\cite{wu2018wasserstein}, OTM~\cite{rout2022generative}, VAEGAN~\cite{larsen2016vaegan}, and VEEGAN~\cite{srivastava2017veegan}. 

\subsection{Some Experimental Details}\label{appendix:details}

All the experiments were implemented using a GPU server with NVIDIA GeForce RTX 4090 GPUs. Each experiment runs on a single GPU. The cost functions, $c_X$ and $c_Y$, are chosen according to \eqref{eq:cost_function_log}. \Cref{tab:parameters} summarizes the chosen regularization constants, learning rates and batch sizes used for training all methods. 

\begin{table}[h!]
    \centering
    \caption{Parameters used in the experiments}
    \begin{tabular}{lllll}
    \hline
    Experiment & Regularization & Batch  & Learning\\
    & Constants & Size & Rates \\
    \hline
    Artificial example & $\lambda_1=10$, & 16 & $\alpha_1=10^{-4}$,\\
    \,  &  $\lambda_2=1$, &  &   $\alpha_2=10^{-4}$ \\
    &  $\lambda_3=5$ & & \\
     \hline
    CIFAR10,  & $\lambda_1=10$,  & 64 & $\alpha_1=2 \cdot 10^{-4}$, \\
    \, TinyImagenet  & $\lambda_2=1$, &  &  $\alpha_2=10^{-4}$ \\
    &  $\lambda_3=5$ & & \\
    \hline
    \end{tabular}
    \label{tab:parameters}
\end{table}

\subsection{Generating Synthetic Gaussian Mixtures}\label{subsec:toy-example}

We examine an artificial dataset designed to elucidate the notions of monotonicity, mode collapse, and sensitivity to parameter initialization. The input dataset $\{x_i\}^n_{i=1}$, with $n=1,000$, is i.i.d.~sampled from a mixture of spherical Gaussians, predominantly centered in the first two coordinates. We consider two different scenarios. In the first scenario, the data points are in $\mathbb{R}^{100}$ and there are 9 Gaussians and in the second one, the data points are in $\mathbb{R}^{500}$ and there are 12 Gaussians. The covariance matrix of each Gaussian is a diagonal matrix with the first two diagonal entries equal to 0.3 and the remaining diagonal entries are 0.003, ensuring that the data distribution can be effectively embedded in $\mathbb{R}^2$. The latent distribution $\nu$ is a single spherical Gaussian in $\mathbb{R}^2$. 

\Cref{fig:toy_ex_a} illustrates samples from the latent distribution and the input data distribution. 
Samples from the latent distribution are colored according to distances from the origin. Samples from the mixture Gaussian distribution are colored according to the underlying Gaussians. 
\begin{figure}[htb]
    \begin{center}
    \begin{minipage}[b]{0.32\textwidth}
        \centering
        \includegraphics[height=0.18\textheight]{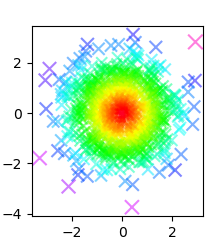}
        \caption*{(a) Latent distribution $\nu$}
    \end{minipage}
    \hfill
    \begin{minipage}[b]{0.32\textwidth}
        \centering
        \includegraphics[height=0.18\textheight]{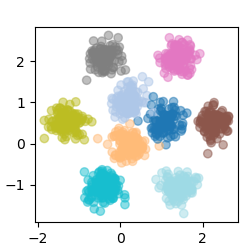}
        \caption*{(b) Dataset in $100$D}
    \end{minipage}
    \hfill
    \begin{minipage}[b]{0.32\textwidth}
        \centering
        \includegraphics[height=0.18\textheight]{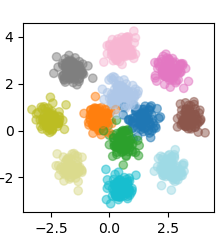}
        \caption*{(c) Dataset in $500$D}
    \end{minipage}
    \caption{
    Illustration of a sample from the latent distribution (left), colored by distances from the origin, an input dataset in $\mathbb{R}^{100}$ with 9 Gaussians (middle) and another input dataset in $\mathbb{R}^{500}$ with 12 Gaussians (right), both colored by Gaussian membership.}
    \label{fig:toy_ex_a}
    \end{center}
\end{figure}

% \vspace{-.5cm}

\Cref{fig:toy_ex_b} displays the outcomes of GMEGAN and the seven baseline models for two input datasets of ambient dimensions $100$ (depicted in rows 1 and 2) and $500$ (depicted in rows 3 and 4).  
Rows 1 and 3 show the first two coordinates of the 1,000 generated samples, where for a given generator $G$, each point $G(y_i)$ is colored according to the distance of $y_i$ from the region in the same way as depicted for samples from $\nu$ in the left panel of \Cref{fig:toy_ex_a}. 
Rows 2 and 4 demonstrate the latent distribution, where each $y_i \in \mathbb{R}^2$ is colored based on the ``cluster'' of $G(y_i)$ using the same color scheme depicted in the center and right panels of \Cref{fig:toy_ex_a}.

\begin{figure*}[htb]
    \begin{minipage}[b]{\linewidth}
        \centering
        \includegraphics[width=1\linewidth]{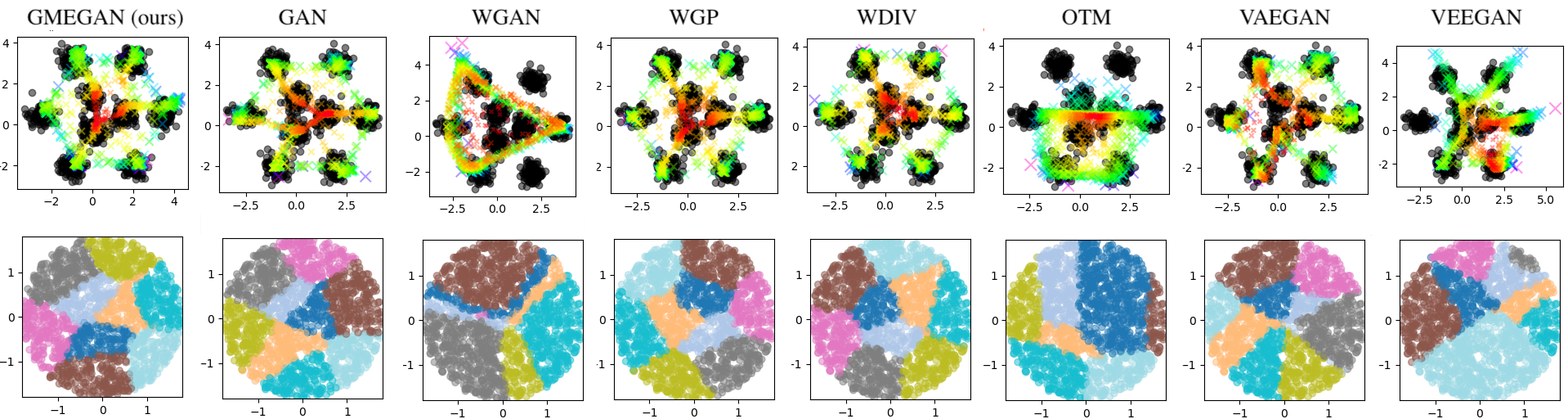}
        \caption*{(a) Results for 9 Gaussians in ${100}$D}
    \end{minipage}
    % \medskip
    \begin{minipage}[b]{\linewidth}
        \centering
        \includegraphics[width=1\linewidth]{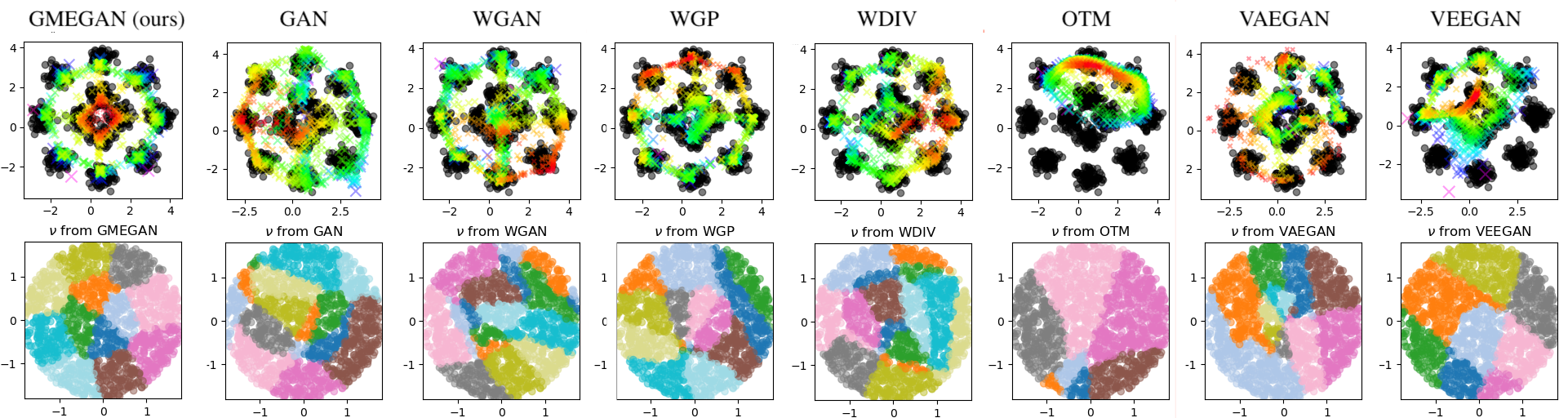}
        \caption*{(b) Results for 12 Gaussians in ${500}$D}
    \end{minipage}
    \hspace{-0.5cm}
    \caption{Results for artificial datasets in 100D and 500D generated by GMEGAN (ours), GAN, WGAN, WGP, WDIV, OTM, VAEGAN, and VEEGAN. Top two rows: 9 Gaussians in $\mathbb{R}^{100}$. Bottom two rows: 12 Gaussians in $\mathbb{R}^{500}$. Rows 1 and 3 show generated samples colored by latent sample location (following the colormap in the left panel of \Cref{fig:toy_ex_a}), with original data shown in black for reference. Rows 2 and 4 display latent representations colored by Gaussian membership, akin to those in the center and right panels of \Cref{fig:toy_ex_a}.}
    \label{fig:toy_ex_b}
\end{figure*}
%\vspace{-3cm}

GMEGAN demonstrates "monotonicity" of its generator, effective latent embedding, and successful mode coverage. While in general, $c_T$-monotonicity is difficult to visualize, the inherently two-dimensional underlying data allows us to approximate $G$ as a map from $\mathbb{R}^2$ to $\mathbb{R}^2$. Consequently, $c_T$-monotonicity and the geometric preservation of the encoder $T$ imply the monotonicity of this approximation of $G$. These ideas are rigorously detailed in Appendix \ref{sec:monotone_G}.

Observing the first and third rows of Figure \ref{fig:toy_ex_b}, we note that the direct monotonicity of the two-dimensional approximation of $G$ is demonstrated by the preservation of the color map in each radial direction, as shown in the left panel of Figure \ref{fig:toy_ex_a}. In contrast, other generators do not exhibit this type of monotonicity for either dataset. It is noteworthy that WGP and WDIV appear to be monotone for the first dataset, where $D=100$.

From the second and fourth rows of this figure, it is evident that GMEGAN provides the best latent representation of the underlying data distribution. It maps the clusters to the latent space while preserving the color arrangement in the center and right panels of \Cref{fig:toy_ex_a}, up to an orthogonal transformation. This highlights the effectiveness of the latent embedding and the generator's monotonicity as well mode coverage. 
In contrast, other methods often fail to capture clusters in the latent space and do not maintain the original color ordering.

Lastly, we observed that GMEGAN's results were insensitive to different parameter initializations, unlike other methods.

\subsection{Generating Real Images}\label{subsec:exp-gen-images}

We conducted experiments by training GMEGAN and several baseline models on two widely used real image datasets: CIFAR10~\cite{cifar10} (comprising 10 image classes) and Tiny ImageNet~\cite{Le2015tinyimagenet} (comprising 200 image classes). Each model was trained for 300 epochs on both datasets until convergence was observed.

To measure the quality of the generated images, we used the Fréchet Inception Distance (FID)~\cite{heusel2017gans}, calculated between the generated samples and the input dataset. We report the average of the 5 lowest FID scores between epochs 250 and 300, using 5,000 randomly sampled real images and 5,000 generated images. The range of 250-300 epochs was chosen because we observed convergence of FID scores for all methods within this range.

To assess mode collapse, we utilized the relative standard deviation of classes, as proposed by~\cite{santurkar2018classification}. For \(K\) classes with \(n_i\) samples generated from the \(i\)-th class, the mean \(\mu\), standard deviation (std) \(\sigma\), and relative std are computed as follows:
\[
    \mu = \sum_{i=1}^K n_i/K, \quad \sigma = \left({\sum_{i=1}^K (n_i - \mu)^2}/{(K-1)}\right)^{1/2} \  \text{and relative std} = \frac{\sigma}{\mu}.
\]
The relative std is calculated after applying a pretrained classifier to the generated examples. A low relative standard deviation indicates well-distributed generated samples across distinct classes. For each GAN-based model, we used the classifier to classify 10,000 generated samples from CIFAR10 and 22,000 from Tiny ImageNet.

To evaluate the robustness of each model, we used 11 distinct neural network architectures for generators, encoders, and discriminators. These architectures varied in factors such as the inclusion or exclusion of batch normalization or fully connected layers, the number of convolutional layers, the number and width of fully connected layers, and the depth of convolutional layers. This variation allowed us to thoroughly assess how these architectural elements influence the effectiveness of GAN-based methods. The box and whisker plots of FID scores and relative standard deviations across the 11 architectures are depicted in \Cref{fig:boxplot-fid} and \Cref{fig:boxplot-sd}, respectively.

\begin{figure}[htb]
\centering
    \begin{subfigure}[CIFAR10]{\includegraphics[width=0.41\textwidth]{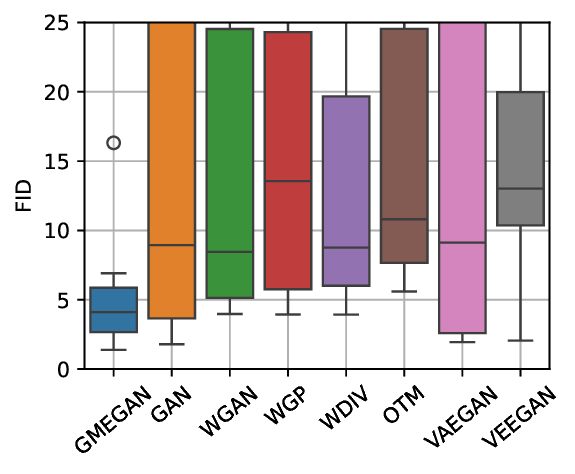}}
    \end{subfigure}
    \begin{subfigure}[Tiny ImageNet]{\includegraphics[width=0.41\textwidth]{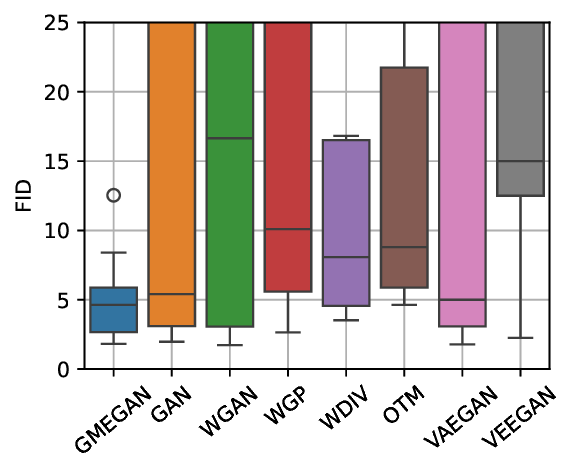}}
    \end{subfigure}
\caption{Box and whisker plots of FID scores obtained from CIFAR10 and Tiny ImageNet using 11 different NN architectures.}
\label{fig:boxplot-fid}
\end{figure}

\begin{figure}[htb]
\centering
    \begin{subfigure}[CIFAR10]{\includegraphics[width=0.41\textwidth]{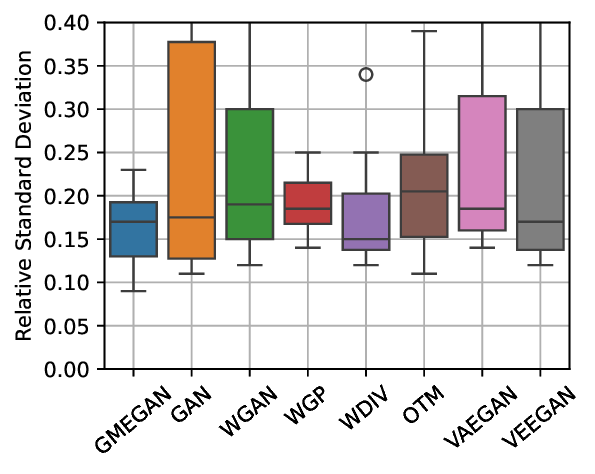}}
    \end{subfigure}
    \begin{subfigure}[Tiny ImageNet]{\includegraphics[width=0.41\textwidth]{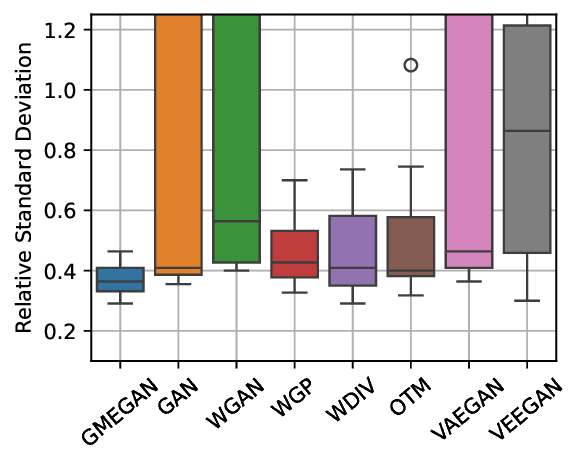}}
    \end{subfigure}
\caption{Box and whisker plots of relative standard deviations obtained from CIFAR10 and Tiny ImageNet datasets using 11 different NN architectures.}
\label{fig:boxplot-sd}
\end{figure}

\begin{figure}[ht!]
    \centering
    \includegraphics[width=0.5\textwidth]{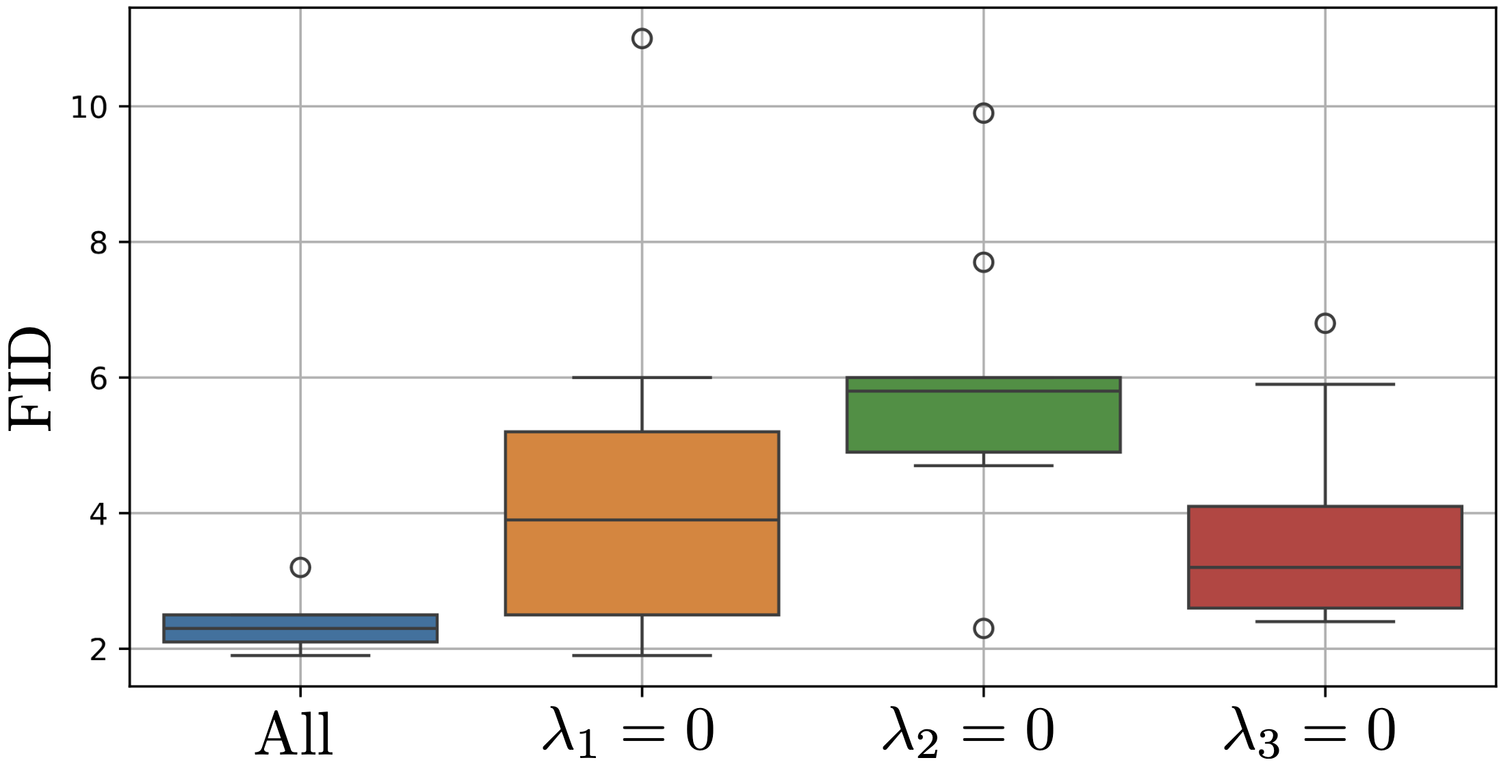}
    \caption{Box and whisker plots of FID scores obtained from GMEGAN (left) and its variants with silent hyperparameters using CIFAR10 with 11 distinct NN architectures.}
    \label{fig:ablation-box-cifar}
\end{figure}

\Cref{fig:boxplot-fid} indicates that GMEGAN outperforms other methods with the lowest mean FID score and, more importantly, the lowest variance among the NN architectures. This suggests that GMEGAN generates the highest quality images and is insensitive to different architectures. \Cref{fig:boxplot-sd} indicates that GMEGAN has the lowest mean relative standard deviation with the lowest variance among the NN architectures. This demonstrates that GMEGAN has the best mode coverage and is the most robust to variations in NN architectures.

Previous works~\cite{goodfellow2016nips,xiang2017effects,lucic2018gans,kurach2019large} have extensively covered the sensitivity of GAN-based and encoder-based GAN models to NN architectures. Surprisingly, GMEGAN exhibits remarkable stability, standing out in contrast to other methods that display significant variance in response to different NN architectures.

\subsection{Ablation Study}\label{sec:exp-abl}

The purpose of this study is to evaluate the impact of different regularization terms on the performance of GMEGAN. \Cref{alg:GMEGAN} incorporates three hyperparameters for various regularization terms in the cost function $\mathcal{L}$, listed in \eqref{eq:total-loss}. Specifically, these are $\lambda_1$ for the GME regularization, $\lambda_2$ for the gradient penalty of the discriminator, and $\lambda_3$ for the reconstruction loss enforcing the inverse relationship $G^{-1} = R_{inv} \circ T$. We conducted experiments on the CIFAR10 dataset to validate the effectiveness of these regularization terms.

\Cref{fig:ablation-box-cifar} compares the FID scores of GMEGAN with three variants, where each of these hyperparameters is set to zero, i.e., $\lambda_i = 0$ for $i=1,2,3$. The results indicate that excluding any of the regularization terms significantly decreases the performance of GMEGAN. 

While all three regularization terms are crucial for the effectiveness of GMEGAN, it appears that $\lambda_3$ may have the least impact on performance. Additionally, removing the GME cost regularization $\lambda_1$ significantly deteriorates the stability of the model.

\section{Conclusion}\label{sec:conclusion}

We introduced a novel method to address key challenges faced by GAN-based approaches, namely mode collapse and training instability. To combat mode collapse, we leveraged a geometry-preserving encoder, which effectively simplifies the minimization problem between two high-dimensional distributions to a lower-dimensional setting. We introduced the new GME cost as a regularization term to enforce a geometry-preserving encoder, and our theoretical analysis guarantees this property.

To address training instability, we utilized the $W_2$ Monge map in the latent space. Our theoretical analysis also guarantees that the resulting generative map is $c_T$-cyclically monotone, where $c_T$ is a cost function defined through the geometry-preserving encoder.

Numerical experiments validate the efficacy of the geometry-preserving encoder and the monotonicity of the generative map. Our approach mitigates mode collapse, demonstrates robustness to parameter initializations and NN architecture selection, and achieves superior performance compared to existing GAN-based (including encoder-based GANs) as evidenced by the lowest FID scores on two common benchmark image datasets.

Future work could explore extending this framework to other types of generative models and further improving computational efficiency.

\subsubsection*{Acknowledgements}
WL acknowledges funding from the National Institute of Standards and Technology (NIST)
under award number 70NANB22H021, YY and DZ acknowledge funding from the Kunshan
Municipal Government research funding, and GL acknowledges funding from NSF award DMS 2124913.

\appendix

\section{Supplemental Details}
We provide some supplemental details to the main text.
\subsection{Proof of  \Cref{prop:upper-bound}}

%\begin{proof}
We first bound $W_p(G_\#\nu, \mu)$ from above
    using the definition of $W_p$ 
as follows:
\begin{align*}
    W_p(G_\#\nu, \mu) &= \left( \min_{\gamma \in \Pi(G_\#\nu,\mu)} \int_{\mathcal{M}^2} \|x-x'\|^p/p  \  d\gamma(x,x')  \right)^{1/p}\\
    &=\left( \min_{\gamma \in \Pi(G_\#\nu,\mu)} \int_{\mathcal{M}^2} \|T^{-1}\circ T(x)-T^{-1}\circ T(x')\|^p/p \  d\gamma(x,x')  \right)^{1/p}\\
    &\leq\frac{1}{\alpha} \left( \min_{\gamma \in \Pi(G_\#\nu,\mu)}  \int_{\mathcal{M}^2} \| T(x)- T(x')\|^p/p \  d\gamma(x,x')  \right)^{1/p}\\
    &\leq\frac{1}{\alpha} \left( \min_{\gamma \in \Pi(T\circ G_\#\nu,T_\#\mu)}  \int_{Y^2} \| y - y'\|^p/p \  d\gamma(y,y')  \right)^{1/p}\\
    &\leq\frac{1}{\alpha} W_p(T\circ G_\#\nu, T_\#\mu).
\end{align*}
Lastly, we bound  $W_p(G_\#\nu, \mu)$  from below as follows:
\begin{align*}
    W_p(G_\#\nu, \mu) 
    &=\left( \min_{\gamma \in \Pi(G_\#\nu,\mu)} \int_{\mathcal{M}^2} \|T^{-1}\circ T(x)-T^{-1}\circ T(x')\|^p/p \  d\gamma(x,x')  \right)^{1/p}\\
    &\geq \alpha \left( \min_{\gamma \in \Pi(G_\#\nu,\mu)}  \int_{\mathcal{M}^2} \| T(x)- T(x')\|^p/p \  d\gamma(x,x')  \right)^{1/p}\\
    &\geq \alpha \left( \min_{\gamma \in \Pi(T\circ G_\#\nu,T_\#\mu)}  \int_{Y^2} \| y - y'\|^p/p \  d\gamma(y,y')  \right)^{1/p}\\
    &\geq \alpha W_p(T\circ G_\#\nu, T_\#\mu).
\end{align*} 
\hfill \qed

\subsection{Clarification of the observed monotonicity in Figure \ref{fig:toy_ex_b}}
\label{sec:monotone_G}
When discussing the special setting of Figure \ref{fig:toy_ex_b} we claimed that the $c_T$-monotonicity of $G$ implies the monotonicity of the approximation of $G$ as a map from $\mathbb{R}^2$ to $\mathbb{R}^2$. We clarify this idea as follows.  

In this setting the manifold $\mathcal{M}$ is a two-dimensional subspace in $\mathbb{R}^D$ ($D=100$ or $500$), which is described by the first two coordinates of the parameterized points in $\mathbb{R}^D$. Therefore, the ideal geometry-preserving encoder is 
\begin{equation}\label{eq:T-gp-ex}
    T(x) = (x_1, x_2) \in \mathbb{R}^2 \ \text{ for } \ x=(x_1, \ldots, x_D) \in \mathbb{R}^D,
\end{equation}
We claim that since the GMEGAN generator $G$ should satisfy $c_T$-cyclical monotonicity, it has to satisfy
\begin{equation}
\langle T \circ G(y) - T \circ G(y'), y - y' \rangle \geq 0 \ \text{ for all } y, y' \in \mathbb{R}^2.
\label{eq:G_monotone}
\end{equation}
To see this claim, apply \Cref{def:c-cm} with $k=2$ to obtain:
\begin{align*}
    c(T(G(y)), y) + c(T(G(y')), y') \leq c(T(G(y)), y')  + c(T(G(y')),y)
\end{align*}
and consequently
\begin{align*}
    -\langle T(G(y)), y\rangle - \langle T(G(y')), y'\rangle \leq -\langle T(G(y)), y'\rangle  - \langle T(G(y')),y\rangle\\
    \langle T (G(y)) - T ( G(y')), y - y' \rangle \geq 0,
\end{align*}
which implies \eqref{eq:G_monotone}.

If the geometry-preserving encoder coincides with $T$  in~\eqref{eq:T-gp-ex}, then \eqref{eq:G_monotone} implies that the map $\tilde{G}(y) = ((G(y))_1, (G(y))_2)$ satisfies the following monotonicity condition  
\[
\langle \tilde{G}(y) - \tilde{G}(y'), y - y' \rangle \geq 0 \ \text{ for all } y, y' \in \mathbb{R}^2.
\]
That is, the first two coordinates of the generator $G$ follow the classical notion of monotonicity. 

% \newpage
% \subsubsection*{References}
\bibliographystyle{plain}
\bibliography{gf}

\begin{thebibliography}{10}

\bibitem{alvarez2018gromov}
David Alvarez-Melis and Tommi Jaakkola.
\newblock {G}romov-{W}asserstein alignment of word embedding spaces.
\newblock pages 1881--1890, October-November 2018.

\bibitem{arjovsky2017towards}
Martin Arjovsky and Leon Bottou.
\newblock Towards principled methods for training generative adversarial
  networks.
\newblock In {\em International Conference on Learning Representations}, 2017.

\bibitem{arjovsky2017wasserstein}
Martin Arjovsky, Soumith Chintala, and L{\'e}on Bottou.
\newblock Wasserstein generative adversarial networks.
\newblock In {\em International conference on machine learning}, pages
  214--223. PMLR, 2017.

\bibitem{pmlr-v70-arora17a}
Sanjeev Arora, Rong Ge, Yingyu Liang, Tengyu Ma, and Yi~Zhang.
\newblock Generalization and equilibrium in generative adversarial nets
  ({GAN}s).
\newblock In Doina Precup and Yee~Whye Teh, editors, {\em Proceedings of the
  34th International Conference on Machine Learning}, volume~70 of {\em
  Proceedings of Machine Learning Research}, pages 224--232. PMLR, 06--11 Aug
  2017.

\bibitem{bao2017cvae}
Jianmin Bao, Dong Chen, Fang Wen, Houqiang Li, and Gang Hua.
\newblock {CVAE-GAN}: fine-grained image generation through asymmetric
  training.
\newblock In {\em Proceedings of the IEEE international conference on computer
  vision}, pages 2745--2754, 2017.

\bibitem{birrell2022structure}
Jeremiah Birrell, Markos Katsoulakis, Luc Rey-Bellet, and Wei Zhu.
\newblock Structure-preserving {GAN}s.
\newblock In {\em International Conference on Machine Learning}, pages
  1982--2020. PMLR, 2022.

\bibitem{bojanowski2018optimizing}
Piotr Bojanowski, Armand Joulin, David Lopez-Pas, and Arthur Szlam.
\newblock Optimizing the latent space of generative networks.
\newblock In {\em International Conference on Machine Learning}, pages
  600--609. PMLR, 2018.

\bibitem{brenier1987}
Y.~Brenier.
\newblock Decomposition polaire et rearrangement monotone des champs de
  vecteurs.
\newblock {\em C. R. Acad. Sci. Paris Ser. I Math.}, 305:805--808, 1987.

\bibitem{bunne2019learning}
Charlotte Bunne, David Alvarez-Melis, Andreas Krause, and Stefanie Jegelka.
\newblock Learning generative models across incomparable spaces.
\newblock In {\em International conference on machine learning}, pages
  851--861. PMLR, 2019.

\bibitem{dumont2022existence}
Th{\'e}o Dumont, Th{\'e}o Lacombe, and Fran{\c{c}}ois-Xavier Vialard.
\newblock On the existence of {Monge} maps for the {Gromov-Wasserstein}
  problem.
\newblock 2022.

\bibitem{fefferman2016testing}
Charles Fefferman, Sanjoy Mitter, and Hariharan Narayanan.
\newblock Testing the manifold hypothesis.
\newblock {\em Journal of the American Mathematical Society}, 29(4):983--1049,
  2016.

\bibitem{gao2020zero}
Rui Gao, Xingsong Hou, Jie Qin, Jiaxin Chen, Li~Liu, Fan Zhu, Zhao Zhang, and
  Ling Shao.
\newblock {Zero-VAE-GAN}: Generating unseen features for generalized and
  transductive zero-shot learning.
\newblock {\em IEEE Transactions on Image Processing}, 29:3665--3680, 2020.

\bibitem{goodfellow2016nips}
Ian Goodfellow.
\newblock {NIPS} 2016 tutorial: Generative adversarial networks.
\newblock {\em arXiv preprint arXiv:1701.00160}, 2016.

\bibitem{goodfellow2014generative}
Ian Goodfellow, Jean Pouget-Abadie, Mehdi Mirza, Bing Xu, David Warde-Farley,
  Sherjil Ozair, Aaron Courville, and Yoshua Bengio.
\newblock Generative adversarial nets.
\newblock In Z.~Ghahramani, M.~Welling, C.~Cortes, N.~Lawrence, and K.Q.
  Weinberger, editors, {\em Advances in Neural Information Processing Systems},
  volume~27. Curran Associates, Inc., 2014.

\bibitem{gropp2020isometric}
Amos Gropp, Matan Atzmon, and Yaron Lipman.
\newblock Isometric autoencoders.
\newblock {\em arXiv preprint arXiv:2006.09289}, 2020.

\bibitem{gulrajani2017improved}
Ishaan Gulrajani, Faruk Ahmed, Martin Arjovsky, Vincent Dumoulin, and Aaron~C
  Courville.
\newblock Improved training of {Wasserstein GANs}.
\newblock {\em Advances in neural information processing systems}, 30, 2017.

\bibitem{heusel2017gans}
Martin Heusel, Hubert Ramsauer, Thomas Unterthiner, Bernhard Nessler, and Sepp
  Hochreiter.
\newblock {GANs} trained by a two time-scale update rule converge to a local
  nash equilibrium.
\newblock {\em Advances in neural information processing systems}, 30, 2017.

\bibitem{ho2020denoising}
Jonathan Ho, Ajay Jain, and Pieter Abbeel.
\newblock Denoising diffusion probabilistic models.
\newblock {\em Advances in neural information processing systems},
  33:6840--6851, 2020.

\bibitem{huang2022improving}
Yan Huang, Tianyuan Zhang, and Huidong Zhu.
\newblock Improving word alignment by adding {Gromov-Wasserstein} into
  attention neural network.
\newblock In {\em Journal of Physics: Conference Series}, volume 2171, page
  012043. IOP Publishing, 2022.

\bibitem{kato2020rate}
Keizo Kato, Jing Zhou, Tomotake Sasaki, and Akira Nakagawa.
\newblock Rate-distortion optimization guided autoencoder for isometric
  embedding in {E}uclidean latent space.
\newblock In Hal~Daumé III and Aarti Singh, editors, {\em Proceedings of the
  37th International Conference on Machine Learning}, volume 119 of {\em
  Proceedings of Machine Learning Research}, pages 5166--5176. PMLR, 13--18 Jul
  2020.

\bibitem{khrulkov2021functional}
Valentin Khrulkov, Artem Babenko, and Ivan Oseledets.
\newblock Functional space analysis of local {GAN} convergence.
\newblock In {\em International Conference on Machine Learning}, pages
  5432--5442. PMLR, 2021.

\bibitem{kingma2013auto}
Diederik~P Kingma and Max Welling.
\newblock Auto-encoding variational bayes.
\newblock In {\em International Conference on Learning Representations}, 2014.

\bibitem{korotin2021wasserstein}
Alexander Korotin, Vage Egiazarian, Arip Asadulaev, Alexander Safin, and Evgeny
  Burnaev.
\newblock Wasserstein-2 generative networks.
\newblock In {\em International Conference on Learning Representations}, 2021.

\bibitem{cifar10}
Alex Krizhevsky, Geoffrey Hinton, et~al.
\newblock Learning multiple layers of features from tiny images.
\newblock 2009.

\bibitem{kurach2019large}
Karol Kurach, Mario Lu{\v{c}}i{\'c}, Xiaohua Zhai, Marcin Michalski, and
  Sylvain Gelly.
\newblock A large-scale study on regularization and normalization in {GANs}.
\newblock In {\em International conference on machine learning}, pages
  3581--3590. PMLR, 2019.

\bibitem{larsen2016vaegan}
Anders Boesen~Lindbo Larsen, S{\o}ren~Kaae S{\o}nderby, Hugo Larochelle, and
  Ole Winther.
\newblock Autoencoding beyond pixels using a learned similarity metric.
\newblock In {\em International conference on machine learning}, pages
  1558--1566. PMLR, 2016.

\bibitem{Le2015tinyimagenet}
Ya~Le and Xuan~S. Yang.
\newblock Tiny imagenet visual recognition challenge.
\newblock 2015.

\bibitem{lee2022regularized}
Yonghyeon Lee, Sangwoong Yoon, MinJun Son, and Frank~C. Park.
\newblock Regularized autoencoders for isometric representation learning.
\newblock In {\em International Conference on Learning Representations}, 2022.

\bibitem{lei2019geometric}
Na~Lei, Kehua Su, Li~Cui, Shing-Tung Yau, and Xianfeng~David Gu.
\newblock A geometric view of optimal transportation and generative model.
\newblock {\em Computer Aided Geometric Design}, 68:1--21, 2019.

\bibitem{li2022gromov}
Xinhang Li, Zhaopeng Qiu, Xiangyu Zhao, Zihao Wang, Yong Zhang, Chunxiao Xing,
  and Xian Wu.
\newblock {Gromov-Wasserstein} guided representation learning for cross-domain
  recommendation.
\newblock In {\em Proceedings of the 31st ACM International Conference on
  Information \& Knowledge Management}, pages 1199--1208, 2022.

\bibitem{liu2019wasserstein2}
Huidong Liu, Xianfeng Gu, and Dimitris Samaras.
\newblock Wasserstein {GAN} with quadratic transport cost.
\newblock In {\em Proceedings of the IEEE/CVF international conference on
  computer vision}, pages 4832--4841, 2019.

\bibitem{lucic2018gans}
Mario Lucic, Karol Kurach, Marcin Michalski, Sylvain Gelly, and Olivier
  Bousquet.
\newblock Are gans created equal? a large-scale study.
\newblock {\em Advances in neural information processing systems}, 31, 2018.

\bibitem{luo2023stabilizing}
Tianjiao Luo, Ziyu Zhu, Jianfei Chen, and Jun Zhu.
\newblock Stabilizing {GAN}s’ training with brownian motion controller.
\newblock In {\em International Conference on Machine Learning}, pages
  23140--23156. PMLR, 2023.

\bibitem{makkuva2020optimal}
Ashok Makkuva, Amirhossein Taghvaei, Sewoong Oh, and Jason Lee.
\newblock Optimal transport mapping via input convex neural networks.
\newblock In {\em International Conference on Machine Learning}, pages
  6672--6681. PMLR, 2020.

\bibitem{memoli2007use}
Facundo M{\'e}moli.
\newblock On the use of {Gromov-Hausdorff} distances for shape comparison.
\newblock In {\em PBG@Eurographics}, 2007.

\bibitem{memoli2011gromov}
Facundo M{\'e}moli.
\newblock {Gromov--Wasserstein} distances and the metric approach to object
  matching.
\newblock {\em Foundations of computational mathematics}, 11:417--487, 2011.

\bibitem{memoli2022distance}
Facundo M{\'e}moli and Tom Needham.
\newblock Distance distributions and inverse problems for metric measure
  spaces.
\newblock {\em Studies in Applied Mathematics}, 149(4):943--1001, 2022.

\bibitem{metz2016unrolled}
Luke Metz, Ben Poole, David Pfau, and Jascha Sohl-Dickstein.
\newblock Unrolled generative adversarial networks.
\newblock {\em arXiv preprint arXiv:1611.02163}, 2016.

\bibitem{nakagawa2023gromovwasserstein}
Nao Nakagawa, Ren Togo, Takahiro Ogawa, and Miki Haseyama.
\newblock {Gromov-Wasserstein} autoencoders.
\newblock In {\em The Eleventh International Conference on Learning
  Representations}, 2023.

\bibitem{peyre2016gromov}
Gabriel Peyr{\'e}, Marco Cuturi, and Justin Solomon.
\newblock {Gromov-Wasserstein} averaging of kernel and distance matrices.
\newblock In {\em International conference on machine learning}, pages
  2664--2672. PMLR, 2016.

\bibitem{pope2021intrinsic}
Phil Pope, Chen Zhu, Ahmed Abdelkader, Micah Goldblum, and Tom Goldstein.
\newblock The intrinsic dimension of images and its impact on learning.
\newblock In {\em International Conference on Learning Representations}, 2021.

\bibitem{rezende2015variational}
Danilo Rezende and Shakir Mohamed.
\newblock Variational inference with normalizing flows.
\newblock In {\em International conference on machine learning}, pages
  1530--1538. PMLR, 2015.

\bibitem{rout2022generative}
Litu Rout, Alexander Korotin, and Evgeny Burnaev.
\newblock Generative modeling with optimal transport maps.
\newblock In {\em International Conference on Learning Representations}, 2022.

\bibitem{santambrogio2015optimal}
Filippo Santambrogio.
\newblock Optimal transport for applied mathematicians.
\newblock {\em Birk{\"a}user, NY}, 55(58-63):94, 2015.

\bibitem{santurkar2018classification}
Shibani Santurkar, Ludwig Schmidt, and Aleksander Madry.
\newblock A classification-based study of covariate shift in {GAN}
  distributions.
\newblock In {\em International Conference on Machine Learning}, pages
  4480--4489. PMLR, 2018.

\bibitem{srivastava2017veegan}
Akash Srivastava, Lazar Valkov, Chris Russell, Michael~U Gutmann, and Charles
  Sutton.
\newblock Veegan: Reducing mode collapse in {GANs} using implicit variational
  learning.
\newblock {\em Advances in neural information processing systems}, 30, 2017.

\bibitem{taghvaei20192}
Amirhossein Taghvaei and Amin Jalali.
\newblock 2-{Wasserstein} approximation via restricted convex potentials with
  application to improved training for {GANs}.
\newblock {\em arXiv preprint arXiv:1902.07197}, 2019.

\bibitem{titouan2019sliced}
Vayer Titouan, R{\'e}mi Flamary, Nicolas Courty, Romain Tavenard, and Laetitia
  Chapel.
\newblock Sliced {Gromov-Wasserstein}.
\newblock {\em Advances in Neural Information Processing Systems}, 32, 2019.

\bibitem{wu2018wasserstein}
Jiqing Wu, Zhiwu Huang, Janine Thoma, Dinesh Acharya, and Luc Van~Gool.
\newblock {Wasserstein} divergence for {GANs}.
\newblock In {\em Proceedings of the European conference on computer vision
  (ECCV)}, pages 653--668, 2018.

\bibitem{xiang2017effects}
Sitao Xiang and Hao Li.
\newblock On the effects of batch and weight normalization in generative
  adversarial networks.
\newblock {\em arXiv preprint arXiv:1704.03971}, 2017.

\bibitem{xu2019scalable}
Hongteng Xu, Dixin Luo, and Lawrence Carin.
\newblock Scalable {Gromov-Wasserstein} learning for graph partitioning and
  matching.
\newblock {\em Advances in neural information processing systems}, 32, 2019.

\bibitem{xu2019gromov}
Hongteng Xu, Dixin Luo, Hongyuan Zha, and Lawrence Carin.
\newblock {Gromov-Wasserstein} learning for graph matching and node embedding.
\newblock In {\em International conference on machine learning}, pages
  6932--6941. PMLR, 2019.

\bibitem{zhu2017unpaired}
Jun-Yan Zhu, Taesung Park, Phillip Isola, and Alexei~A Efros.
\newblock Unpaired image-to-image translation using cycle-consistent
  adversarial networks.
\newblock In {\em Proceedings of the IEEE international conference on computer
  vision}, pages 2223--2232, 2017.

\end{thebibliography}

\end{document}